\documentclass[11pt,a4paper]{article}

\usepackage[margin=25mm]{geometry}
\usepackage[T1]{fontenc}

\usepackage{amsmath,amssymb,amsthm}
\allowdisplaybreaks

\newtheorem{corollary}{Corollary}[section]
\newtheorem{proposition}{Proposition}[section]

\theoremstyle{remark}
\newtheorem{remark}{Remark}[section]
\theoremstyle{plain}

\usepackage{graphicx}
\usepackage{array,tabularx,tabularray}
\usepackage{booktabs,longtable,rotating,dcolumn}
\usepackage{xcolor,colortbl}
\usepackage{caption}
\usepackage[nospace]{cite}
\usepackage{url}
\usepackage[hidelinks,bookmarksnumbered]{hyperref}

\newcommand{\ack}[1]{\section*{Acknowledgements}#1}

\title{Dissipative Lagrangian Neural Networks for the
Interpretable Modelling of Passive Mechanical Systems}

\author{Rob Seabourne and Alice Cicirello\\
\small Department of Engineering, University of Cambridge\\
\small Trumpington Street, Cambridge CB2 1PZ, UK\\[4pt]
\small\texttt{\href{mailto:rs2229@cantab.ac.uk}{rs2229@cantab.ac.uk}}}

\date{}
\begin{document}

\maketitle

\begin{abstract}
\noindent In this paper, we propose the Dissipative Lagrangian Neural Network (dLNN), a continuous-time system identification technique that separately learns the conservative and dissipative dynamics of a mechanical system directly from data. A key contribution is a novel parametrisation of dissipative forces that (i) is structurally constrained to be dissipative, enforcing passivity of the learned model, (ii) cannot absorb conservative mechanics, and (iii) retains sufficient expressivity to represent a broad family of state-dependent damping and regularised friction forces. The model is validated on prediction tasks on two real-world mechanical systems: trajectory prediction of a chaotic double pendulum and inverse-dynamics modelling of a six-degree-of-freedom industrial robot. Across these tasks, the dLNN demonstrates improved predictive performance over existing Lagrangian architectures while remaining competitive with a black-box neural model. The imposed physical structure also improves generalisation to unseen states, with the benefits becoming particularly pronounced when training data is limited. Finally, we introduce a symbolic distillation procedure that extracts closed-form expressions from the learned physical components. Applied to the industrial robot, this produces interpretable symbolic models that retain high predictive accuracy while reducing the computational overhead associated with evaluating neural Lagrangian models.
\end{abstract}

\section{Introduction}
Building accurate numerical models of mechanical systems is essential for many different areas of engineering. From model-based control strategies to building digital twins for robotic learning or structural health monitoring, the availability of accurate physical models is essential to the efficient engineering of complex built-up mechanical systems. 

System identification provides a means of constructing dynamical models from measured data. Traditionally, this involves prescribing a model structure and estimating its parameters from observations, as in linear ARX models~\cite{ljung_system_1999} and their nonlinear extension, NARX models~\cite{leontaritis_input-output_1985}. More modern approaches seek to provide additional insight into the dynamics by quantifying predictive uncertainty or discovering interpretable governing equations. For example, Gaussian process models provide nonlinear predictions together with uncertainty estimates conditional on the model assumptions~\cite{kocijan_dynamic_2005}, while sparse identification of nonlinear dynamics (SINDy) uses sparse regression to select terms from a candidate-function library, yielding compact, interpretable equations of motion~\cite{brunton_discovering_2016}.

However, in recent years, the increased availability of data and computing power has also supported the development of flexible neural-network models for nonlinear system identification. Discrete-time approaches include recurrent neural networks (RNNs), their long short-term memory (LSTM) variants, and Transformer-based architectures, which learn temporal relationships from sequences of observations to predict future outputs~\cite{elman_finding_1990,hochreiter_long_1997,wen_transformers_2023}. For physical systems (which are typically modelled by continuous-time differential equations), Neural ODEs offer a natural alternative by learning the underlying continuous-time vector field of a system~\cite{chen_neural_2018}. Although these approaches can capture complex nonlinear behaviour, predictive accuracy alone does not ensure that the learned dynamics respect the physical properties of the underlying mechanical system. Typically, in engineering applications, there is already a wealth of knowledge about the systems of interest that should be incorporated into the system identification procedure.

Physics-enhanced machine learning (PEML) bridges first-principles and black-box approaches by combining physical knowledge with data-driven modelling~\cite{cicirello_physics-enhanced_2024,haywood-alexander_discussing_2024}. This can improve data efficiency, generalisation and the physical realism of predictions. PEML encompasses a range of approaches, including learning a neural-network discrepancy to a first-principles model~\cite{rackauckas_universal_2021,lai_structural_2021}, incorporating physical knowledge into Gaussian process regression~\cite{cross_spectrum_2024}, imposing physical constraints during sparse equation discovery~\cite{loiseau_constrained_2018,lathourakis_physics_2024}, enforcing governing equations through the training loss~\cite{raissi_physics-informed_2019}, and embedding physical structure directly in the model architecture~\cite{greydanus_hamiltonian_2019,lutter_deep_2019,cranmer_lagrangian_2020}. Within this broader family, physics-informed machine learning (PIML) refers specifically to approaches that constrain the structure of data-driven models through biases derived from physics and domain knowledge~\cite{cicirello_physics-enhanced_2024}.

A family of PIML models that use the Lagrangian or Hamiltonian formalism has been developed to learn physically realistic dynamics, known respectively as Lagrangian Neural Networks~\cite{cranmer_lagrangian_2020,lutter_deep_2019} and Hamiltonian Neural Networks~\cite{greydanus_hamiltonian_2019}. These models naturally embed physical conservative structure into their learned dynamics, which has been shown to improve modelling performance on conservative dynamical systems such as multi-body gravitational systems~\cite{greydanus_hamiltonian_2019}, robotic arms~\cite{lutter_deep_2019}, vibrations~\cite{bacsa_symplectic_2023}, and relativistic particle dynamics~\cite{cranmer_lagrangian_2020}. The key distinction between the Lagrangian and Hamiltonian formalisms is their choice of state variables: the Lagrangian formulation is expressed in terms of generalised coordinates and velocities $(\mathbf{q},\dot{\mathbf{q}})$, whereas the Hamiltonian formulation is expressed in terms of generalised coordinates and their conjugate canonical momenta $(\mathbf{q},\mathbf{p})$. Typically, in mechanical systems of interest, the expression of the canonical momenta $\mathbf{p}$ is non-trivial, especially as the complexity of the system increases. On the other hand, generalised velocities $\dot{\mathbf{q}}$ are usually readily estimated. LNNs are therefore particularly well suited to mechanical modelling, but their conservative formulation neglects the energy losses inevitably present in real-world systems. To account for these losses, this work proposes an extension to the LNN that models dissipation alongside conservative dynamics.

Various alternative black-box approaches have been proposed to attempt to extend LNNs and HNNs to non-conservative systems. These typically augment the conservative mechanics by modelling either (i) a black-box generalised force $\mathbf{Q}_\theta$ \cite{xiao_generalized_2025,wu_dynamic_2024}, or (ii) a black-box Rayleigh dissipation function $\mathcal{R}_\theta$~\cite{sosanya_dissipative_2022}. Although highly expressive, these techniques allow the non-conservative branch to absorb dynamics that should be captured by the conservative model, weakening the physical interpretability of the model components. Furthermore, the learned non-conservative forces are not constrained to be dissipative, so the passivity of the original LNN formulation may be lost. Solutions have been proposed~\cite{wu_dynamic_2024} that penalise the output of the non-conservative network in the loss function to lessen this issue. However, this does not structurally guarantee that the learned non-conservative dynamics remain separate from the conservative dynamics, allowing the non-conservative branch to absorb dynamics that should instead be represented by the Lagrangian.

Alternative approaches incorporate dissipation with a  more structured approach. Examples include a quadratic Rayleigh dissipation function $\mathcal{R}_\theta=\frac{1}{2}\dot{\mathbf{q}}^\mathsf{T}\mathbf{R}_\theta(\mathbf{q})\dot{\mathbf{q}}$~\cite{hansen_learning_2025} or a closed-form friction model with learnable parameters~\cite{lutter_deep_2019-1}. These methods guarantee that the model cannot absorb conservative mechanics and passivity can be preserved. However, these methods have limited expressivity. Dissipation in complex engineering systems rarely follows a simple parametric form. Only in certain regions of parameter space and under certain loading conditions might simple parametric dissipation models be indistinguishable from the ground-truth complex dissipation function~\cite{lathourakis_physics_2024}.

In this paper, we introduce the Dissipative Lagrangian Neural Network or dLNN. Within the proposed approach, a parametrisation
$\mathbf{Q}^{\mathrm{diss}}(\mathbf{q},\dot{\mathbf{q}})
=
-\mathbf{D}(\mathbf{q},\dot{\mathbf{q}})\dot{\mathbf{q}}$,
with $\mathbf{D}(\mathbf{q},\dot{\mathbf{q}})\succ 0$, is learned alongside the conservative dynamics to model dissipative forces. This parametrisation is chosen as it structurally prohibits the absorption of conservative dynamics, enforces passivity, and is representative of all dissipative forces $\mathbf{Q}^{\mathrm{diss}}(\mathbf{q},\dot{\mathbf{q}})$ satisfying $\dot{\mathbf{q}}^{\mathsf T}\mathbf{Q}^{\mathrm{diss}}<0$ for $\dot{\mathbf{q}}\neq\mathbf{0}$ and $\mathbf{Q}^{\mathrm{diss}}(\mathbf{q},\mathbf{0})=\mathbf{0}$. These conditions encompass a broad class of dissipative effects encountered in mechanical systems, including viscous and nonlinear damping and regularised friction models~\cite{olsson_friction_1998,al-hababi_critical_2020}. This dissipative extension of the Lagrangian Neural Network naturally incorporates the friction and damping that are inevitable in the real world, while preserving the interpretability, structure and passivity of the original architecture.

The proposed dLNN is validated against other state-of-the-art models for the unforced trajectory prediction of a real-world chaotic double pendulum and for the inverse-dynamics problem of a real industrial robot. It is shown that the dLNN has strong predictive performance in both cases.

Inspired by the works~\cite{cranmer_discovering_2020,lai_structural_2021}, this paper also introduces a symbolic distillation method based around the SINDy method~\cite{brunton_discovering_2016} that uses the separability property of the dLNN to extract a symbolic mechanical model for the purposes of computational efficiency and interpretability. This technique permits the creation of full symbolic models and hybrid symbolic-neural network models that retain the good prediction accuracy of the dLNN model whilst reducing computational evaluation time and increasing interpretability.

\section{Dissipative Lagrangian Neural Networks}

\subsection{Lagrangian mechanics background}

A time-independent Lagrangian $\mathcal{L}(\mathbf q,\dot{\mathbf q})$ is a scalar function describing the dynamics of a system in terms of the generalised coordinates $\mathbf q$ and generalised velocities $\dot{\mathbf q}$~\cite{goldstein_classical_2008}. In the presence of non-conservative generalised forces $\mathbf Q$, the equations of motion can be expressed by the forced Euler-Lagrange equations
\begin{equation}
\label{eq:EL-EQs}
\frac{\mathrm d}{\mathrm dt}
\left(
\frac{\partial\mathcal{L}}
{\partial\dot{\mathbf q}}
\right)
-
\frac{\partial\mathcal{L}}
{\partial\mathbf q}
=
\mathbf Q.
\end{equation}
Expanding the total derivative using the chain rule gives
\begin{equation}
\frac{\partial^2\mathcal{L}}
{\partial\dot{\mathbf q}\partial\dot{\mathbf q}}
\ddot{\mathbf q}
+
\frac{\partial^2\mathcal{L}}
{\partial\dot{\mathbf q}\partial\mathbf q}
\dot{\mathbf q}
-
\frac{\partial\mathcal{L}}
{\partial\mathbf q}
=
\mathbf Q.
\end{equation}
For finite-dimensional classical mechanical systems whose kinetic energy $T$ is quadratic in the generalised velocities, the Lagrangian can be written in the form
\begin{equation}
\mathcal{L}(\mathbf q,\dot{\mathbf q})
=
T-V
=
\frac12
\dot{\mathbf q}^{\mathsf T}
\mathbf M(\mathbf q)
\dot{\mathbf q}
-
V(\mathbf q),
\end{equation}
where $\mathbf M(\mathbf q)$ is the symmetric positive-definite mass matrix and $V(\mathbf q)$ is the potential energy~\cite{lynch_modern_2017}. Substituting this form into the Euler-Lagrange equations gives
\begin{equation}
\mathbf M(\mathbf q)\ddot{\mathbf q}
+
\mathbf c(\mathbf q,\dot{\mathbf q})
+
\nabla_\mathbf{q}V
=
\mathbf Q,
\end{equation}
where the Coriolis and centrifugal vector is defined by
\begin{equation}
\label{eq:corio}
\mathbf c(\mathbf q,\dot{\mathbf q})
=
\dot{\mathbf M}(\mathbf q)\dot{\mathbf q}
-
\frac12
\nabla_\mathbf{q}
\left(
\dot{\mathbf q}^{\mathsf T}
\mathbf M(\mathbf q)
\dot{\mathbf q}
\right).
\end{equation}

\begin{figure}[h]
    \centering
    \includegraphics[width=0.8\linewidth]{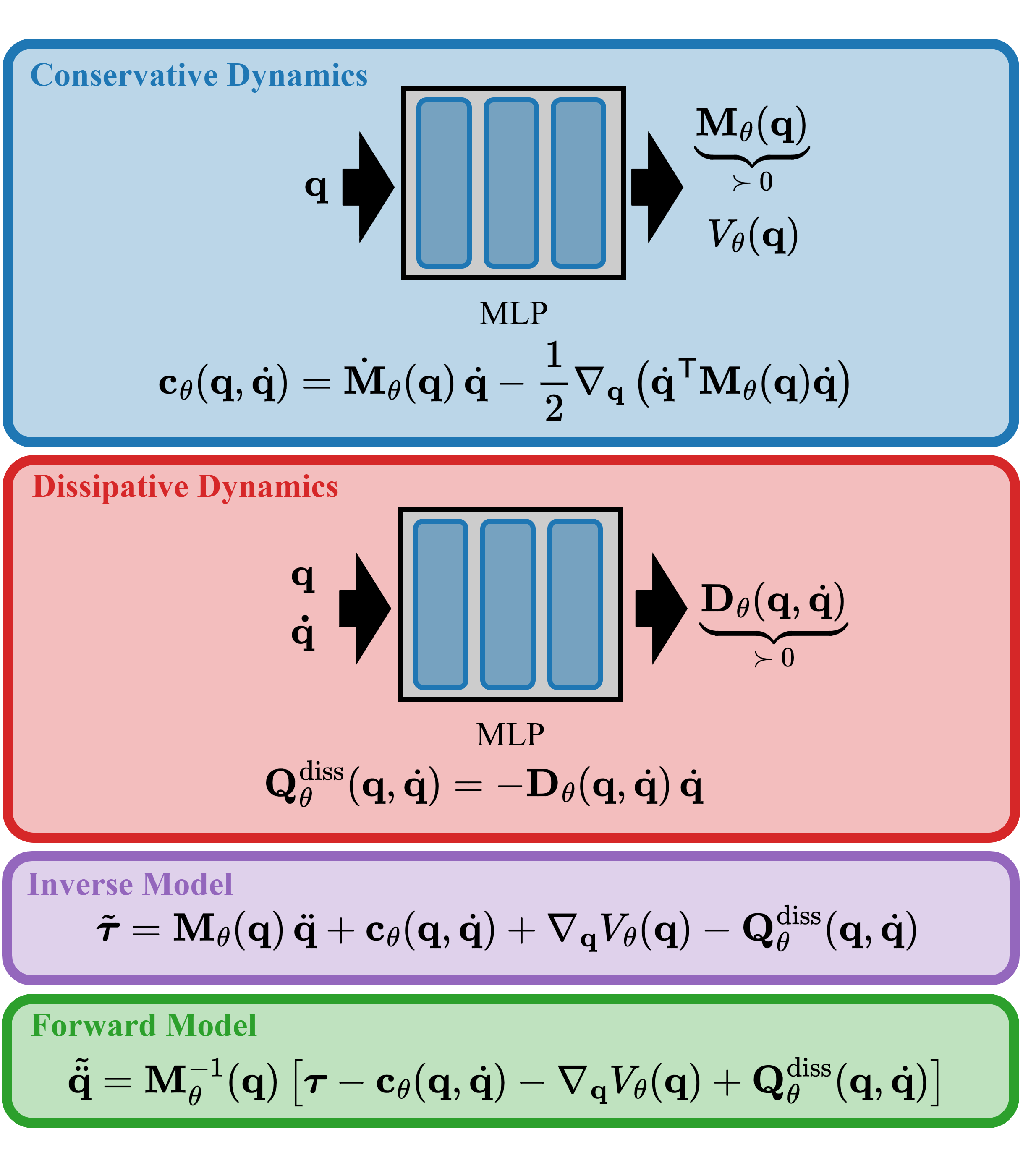}
    \caption{Architecture of the Dissipative Lagrangian Neural Network (dLNN). The conservative network learns the positive-definite mass matrix $\mathbf M_\theta(\mathbf q)$ and potential energy $V_\theta(\mathbf q)$, while the dissipative network learns the positive-definite dissipation matrix $\mathbf D_\theta(\mathbf q,\dot{\mathbf q})$. These components define a common physical model that can be evaluated for either forward or inverse dynamics.}
    \label{fig:placeholder}
\end{figure}

\subsection{Forward and inverse dynamics}
Typically in mechanics applications, the Lagrangian $\mathcal{L}$ and generalised forces $\mathbf{Q}$ are symbolic functions that are derived from first principles and the Euler-Lagrange equations, Eq.~\eqref{eq:EL-EQs}, are evaluated analytically to generate equations of motion. Lagrangian Neural Networks~\cite{cranmer_lagrangian_2020,lutter_deep_2019} introduced the idea that the Lagrangian can be replaced by a black-box neural network $\mathcal{L}_\theta(\mathbf{q},\dot{\mathbf{q}})$, and the relevant terms in the Euler-Lagrange equations can be evaluated numerically using automatic differentiation to generate the acceleration field of a physical system. To extend this idea to dissipative systems, this paper proposes learning a strictly dissipative force $\mathbf{Q}^\mathrm{diss}_\theta(\mathbf{q},\dot{\mathbf{q}})$ alongside the conservative $\mathcal{L}_\theta(\mathbf{q},\dot{\mathbf{q}})=\frac12
\dot{\mathbf q}^{\mathsf T}
\mathbf M_\theta(\mathbf q)
\dot{\mathbf q}
-
V_\theta(\mathbf q)$.

The dynamical model can be evaluated and trained on either forward or inverse dynamics, depending on the specific prediction task, and the availability of data. For mechanical systems, forward dynamics involves the prediction of future behaviour, given the current coordinates and velocities $\mathbf{q}_t,\dot{\mathbf{q}}_t$ and input forces $\boldsymbol{\tau}$ if any exist. Forward dynamical models are widely used in engineering for simulation and model-based control, where they predict the evolution of the system under applied inputs~\cite{lynch_modern_2017}. A forward model will directly model the relationship $(\mathbf{q}_t,\dot{\mathbf{q}}_t,(\boldsymbol{\tau}_t)\rightarrow\ddot{\mathbf{q}}_t)$, which can be used alongside a numerical integrator to obtain trajectory predictions $(\mathbf{q}_t,\dot{\mathbf{q}}_t,(\boldsymbol{\tau}_{t:T})\rightarrow \mathbf{q}_{t+1:T},\dot{\mathbf{q}}_{t+1:T})$. For mechanical systems, the forward model is defined from $V_\theta(\mathbf{q})$, $\mathbf{M}_\theta(\mathbf{q})$ and $\mathbf{Q}^\mathrm{diss}_\theta(\mathbf{q},\dot{\mathbf{q}})$ as

\begin{equation}
\label{eq:forward_dyn}
    \widetilde{\ddot{\mathbf{q}}}
    =
    \mathbf{M}^{-1}_\theta(\mathbf{q})
    \left[
        \boldsymbol{\tau}
        +
        \mathbf{Q}^{\mathrm{diss}}_\theta(\mathbf{q},\dot{\mathbf{q}})
        -
        \mathbf{c}_\theta(\mathbf{q},\dot{\mathbf{q}})
        -
        \nabla_{\mathbf{q}}V_\theta(\mathbf{q})
    \right].
\end{equation}
Alternatively, inverse-dynamics models attempt to reconstruct the applied forces necessary for a given trajectory, that is the relationship $(\mathbf{q}_t,\dot{\mathbf{q}}_t,\ddot{\mathbf{q}}_t\rightarrow \boldsymbol{\tau}_t)$. Inverse-dynamics models have a more specialised range of applications than forward-dynamics models, primarily being used to determine the forces or torques required to realise a prescribed motion, especially in robotics~\cite{lynch_modern_2017}. From the same $V_\theta(\mathbf{q})$, $\mathbf{M}_\theta(\mathbf{q})$ and $\mathbf{Q}^\mathrm{diss}_\theta(\mathbf{q},\dot{\mathbf{q}})$, the inverse model is defined as
\begin{equation}
    \widetilde{\boldsymbol{\tau}}
    =
    \mathbf{M}_\theta(\mathbf{q})\ddot{\mathbf{q}}
    +
    \mathbf{c}_\theta(\mathbf{q},\dot{\mathbf{q}})
    +
    \nabla_{\mathbf{q}}V_\theta(\mathbf{q})
    -
    \mathbf{Q}^{\mathrm{diss}}_\theta(\mathbf{q},\dot{\mathbf{q}}).
\end{equation}

\subsection{Learning conservative mechanics}
\label{sec:Learning Conservative Mechanics}

The conservative component of the dynamics is parametrised by learning the mass matrix
$\mathbf{M}_\theta(\mathbf q)$ and potential energy $V_\theta(\mathbf q)$ as the outputs of a neural network, giving the structured Lagrangian
\begin{equation}
    \mathcal{L}_\theta(\mathbf q,\dot{\mathbf q})
    =
    \frac{1}{2}\dot{\mathbf q}^{\mathsf T}
    \mathbf{M}_\theta(\mathbf q)
    \dot{\mathbf q}
    -
    V_\theta(\mathbf q).
\end{equation}
For a classical mechanical system, the mass matrix must be symmetric positive definite~\cite{lynch_modern_2017}. This is enforced by parametrising it through a lower-triangular Cholesky factor,
\begin{equation}
\label{eq:MASS_MATRIX}
    \mathbf{M}_\theta(\mathbf q)
    =
    \mathbf{L}^{(M)}_\theta(\mathbf q)
    \mathbf{L}^{(M)\mathsf T}_\theta(\mathbf q)
    +
    \epsilon \mathbf I,
    \qquad \epsilon > 0,
\end{equation}
where
\begin{equation}
\mathbf{L}^{(M)}_\theta(\mathbf q)
=
\begin{bmatrix}
\operatorname{softplus}\!\left(\tilde{\ell}_{11,\theta}(\mathbf q)\right)
    & 0 & \cdots & 0 \\
\tilde{\ell}_{21,\theta}(\mathbf q)
    & \operatorname{softplus}\!\left(\tilde{\ell}_{22,\theta}(\mathbf q)\right)
    & \ddots & \vdots \\
\vdots
    & \ddots
    & \ddots & 0 \\
\tilde{\ell}_{n1,\theta}(\mathbf q)
    & \cdots
    & \tilde{\ell}_{n,n-1,\theta}(\mathbf q)
    & \operatorname{softplus}\!\left(\tilde{\ell}_{nn,\theta}(\mathbf q)\right)
\end{bmatrix}.
\end{equation}
The unconstrained quantities $\tilde{\ell}_{ij,\theta}(\mathbf q)$ are outputs of a neural network. Applying a softplus activation to the diagonal entries ensures that $\mathbf{L}^{(M)}_\theta(\mathbf q)$ has strictly positive diagonal elements and hence that
$\mathbf{L}^{(M)}_\theta\mathbf{L}^{(M)\mathsf T}_\theta\succ\mathbf 0$.
The additional term $\epsilon\mathbf I$ provides a lower bound on the eigenvalues of the learned mass matrix and improves numerical conditioning when solving the forward dynamics for $\ddot{\mathbf q}$. The potential energy is similarly parametrised as
\begin{equation}
    V_\theta(\mathbf q)
    =
    \operatorname{softplus}\!\left(\tilde V_\theta(\mathbf q)\right),
\end{equation}
where $\tilde V_\theta(\mathbf q)$ is the unconstrained output of a neural network. This ensures that the learned potential is bounded below. Since the equations of motion depend only on derivatives of the potential, they are invariant to an arbitrary additive constant in $V(\mathbf q)$. Therefore, for the lower-bounded potentials considered here, fixing the energy datum such that $V_\theta(\mathbf q)>0$ does not restrict the resulting dynamics. Together with $\mathbf{M}_\theta(\mathbf q)\succ\mathbf 0$, this guarantees that the learned mechanical energy is bounded below, a property required for the passivity result introduced in Section~\ref{sec:PASS}.

\subsection{Learning dissipative mechanics} To model non-conservative mechanics, a distinction is made between externally applied input forces $\boldsymbol{\tau}$, which can be assumed known in a forward-dynamics task, and dissipative forces such as friction, damping, or drag, that are considered unknown. The latter are modelled as state-dependent generalised forces $\mathbf{Q}^{\mathrm{diss}}(\mathbf{q},\dot{\mathbf{q}})$. The total non-conservative generalised force is therefore written as \begin{equation} \mathbf{Q} = \boldsymbol{\tau} + \mathbf{Q}^{\mathrm{diss}}(\mathbf{q},\dot{\mathbf{q}}). \end{equation} The dissipative component $\mathbf{Q}^{\mathrm{diss}}_{\theta}(\mathbf{q},\dot{\mathbf{q}})$ is learned by a second, non-conservative neural network. It is parametrised as \begin{equation} \mathbf{Q}^{\mathrm{diss}}_{\theta} (\mathbf{q},\dot{\mathbf{q}}) = -\mathbf{D}_{\theta} (\mathbf{q},\dot{\mathbf{q}}) \dot{\mathbf{q}}, \qquad \mathbf{D}_{\theta} (\mathbf{q},\dot{\mathbf{q}}) \succ \mathbf{0}. \end{equation} As for the mass matrix, positive definiteness of $\mathbf{D}_{\theta}(\mathbf{q},\dot{\mathbf{q}})$ is enforced using a Cholesky parametrisation, \begin{equation} \mathbf{D}_{\theta} (\mathbf{q},\dot{\mathbf{q}}) = \mathbf{L}^{(D)}_{\theta} (\mathbf{q},\dot{\mathbf{q}}) \mathbf{L}^{(D)\mathsf T}_{\theta} (\mathbf{q},\dot{\mathbf{q}}), \end{equation} with strictly positive diagonal entries in $\mathbf{L}^{(D)}_{\theta}$ enforced using softplus activations, as described previously for $\mathbf{L}^{(M)}_{\theta}$. This construction guarantees 
\begin{equation}
\label{eq:DISS FORCE PROP1}
\dot{\mathbf{q}}^{\mathsf T} \mathbf{Q}^{\mathrm{diss}}_{\theta} < 0 \qquad \text{for } \dot{\mathbf{q}}\neq\mathbf{0}, \end{equation}
and
\begin{equation}
\label{eq:DISS FORCE PROP2}
\mathbf{Q}^{\mathrm{diss}}_{\theta} (\mathbf{q},\mathbf{0}) = \mathbf{0}.
\end{equation}
These properties are sufficient to guarantee passivity in the learned dynamics and to ensure $\mathbf{Q}^{\mathrm{diss}}_\theta$ cannot absorb conservative dynamics that should be explained by $\mathbf{M}_\theta$ or $V_\theta$, as explained in further detail in Section~\ref{sec:THEORY}. 
In fact, the decomposition of $\mathbf{Q}(\mathbf{q},\dot{\mathbf{q}})=-\mathbf{D}(\mathbf{q},\dot{\mathbf{q}})\dot{\mathbf{q}}$ can be shown to be representative of all forces satisfying Eqs.~\eqref{eq:DISS FORCE PROP1} and~\eqref{eq:DISS FORCE PROP2} as proved in Section~\ref{sec:express}. 

This parametrisation can therefore represent a broad family of instantaneous dissipative forces, including linear viscous damping, nonlinear velocity-dependent damping (such as quadratic drag), configuration-dependent damping, and coupled damping between degrees of freedom. It also accommodates regularised Coulomb friction and Stribeck-type friction, provided the resulting force opposes motion and vanishes at zero velocity. Unfortunately, dissipative phenomena with memory, such as frictional hysteresis and stiction, require additional internal states to be represented~\cite{olsson_friction_1998} and cannot be represented by this parametrisation as described.

\subsection{Optional Coulomb-like friction augmentation}
\label{sec:COULOMB EXTRA}
Although the parametrisation $\mathbf{Q}(\mathbf{q},\dot{\mathbf{q}})=-\mathbf{D}(\mathbf{q},\dot{\mathbf{q}})\dot{\mathbf{q}}$ is, in principle, capable of representing discontinuous dissipative forces, this may require $\mathbf{D}$ itself to be discontinuous or singular. When $\mathbf{D}$ is represented by a finite neural network with continuous activation functions, such behaviour can only be approximated, potentially making sharp discontinuities such as Coulomb friction difficult to learn accurately. To address this, optionally a branch can be added to $\mathbf{Q}^{\mathrm{diss}}_\theta$ to approximate this discontinuity
\begin{equation}
Q_{\theta,i}^{\mathrm{C}}
=
-\underbrace{f_{\theta,i}(\mathbf{q},\dot{\mathbf{q}})}_{>0}
\tanh(k\dot{q}_i),
\end{equation}
where the components of $\mathbf{f}_\theta(\mathbf{q},\dot{\mathbf{q}})$ are additional outputs of the dissipative network that are constrained to be positive by a softplus output activation, and where $k>0$ controls the sharpness of the transition, with sufficiently large $k$ providing a smooth approximation to the Coulomb friction discontinuity. This retains the properties of separability and passivity, as explained in Sections~\ref{sec:PASS} and~\ref{sec:SEPER}. Separability of the terms is ensured since $\mathbf{Q}^\mathrm{C}_\theta(\mathbf{q},\mathbf{0})=\mathbf{0}$ and passivity is retained since
\begin{equation}
    \dot{\mathbf{q}}^{\mathrm{T}}
    \mathbf{Q}_{\theta}^{\mathrm{C}}
    =
    -\sum_i
    \underbrace{f_{\theta,i}(\mathbf{q},\dot{\mathbf{q}})}_{>0}
    \underbrace{\dot{q}_i\tanh(k\dot{q}_i)}_{>0\;\text{if}\;\dot{q}_i\neq 0}
    < 0,
    \qquad
    \dot{\mathbf{q}}\neq\mathbf{0}.
\end{equation}
Note this term does not replace the parametrisation $-\mathbf{D}_\theta(\mathbf{q},\dot {\mathbf{q}})\dot{\mathbf{q}}$ but rather complements it in modelling particularly rich discontinuous friction. This term on its own does not have the expressivity of the standard formulation. The total dissipative force is then
\begin{equation}
    \mathbf{Q}^{\mathrm{diss}}_\theta(\mathbf{q},\dot {\mathbf{q}})=-\mathbf{D}_\theta(\mathbf{q},\dot {\mathbf{q}})\dot{\mathbf{q}}+\mathbf{Q}_{\theta}^{\mathrm{C}}(\mathbf{q,\dot{\mathbf{q}}})
\end{equation}

\subsection{Training objectives}
The learned components $\mathbf{Q}^{\mathrm{diss}}_\theta$, $V_\theta$, and $\mathbf{M}_\theta$ can then be used to approximate either the forward dynamics $(\mathbf{q}_t,\dot{\mathbf{q}}_t,(\boldsymbol{\tau}_t) \rightarrow \ddot{\mathbf{q}}_t)$ or the inverse dynamics $(\mathbf{q}_t,\dot{\mathbf{q}}_t,\ddot{\mathbf{q}}_t \rightarrow \boldsymbol{\tau}_t)$. For forward dynamics, the predicted acceleration is given by
\begin{equation}
    \widetilde{\ddot{\mathbf{q}}}
    =
    \mathbf{M}^{-1}_\theta(\mathbf{q})
    \left[
        \boldsymbol{\tau}
        +
        \mathbf{Q}^{\mathrm{diss}}_\theta(\mathbf{q},\dot{\mathbf{q}})
        -
        \mathbf{c}_\theta(\mathbf{q},\dot{\mathbf{q}})
        -
        \nabla_{\mathbf{q}}V_\theta(\mathbf{q})
    \right].
\end{equation}
Similarly, for inverse dynamics, the applied generalised forces are predicted as
\begin{equation}
    \widetilde{\boldsymbol{\tau}}
    =
    \mathbf{M}_\theta(\mathbf{q})\ddot{\mathbf{q}}
    +
    \mathbf{c}_\theta(\mathbf{q},\dot{\mathbf{q}})
    +
    \nabla_{\mathbf{q}}V_\theta(\mathbf{q})
    -
    \mathbf{Q}^{\mathrm{diss}}_\theta(\mathbf{q},\dot{\mathbf{q}}).
\end{equation}
In the case of forward dynamics, the models may be trained using either instantaneous acceleration data or trajectory data. When acceleration data is available, the training loss is simply the weighted mean-squared error between the predicted and measured accelerations,
\begin{equation}
\label{eq:accel_training}
    \mathcal{J}_{\mathrm{acc}}
    =
    \frac{1}{N}
    \sum_{j=1}^{N}
    \left\|
        \widetilde{\ddot{\mathbf{q}}}^{(j)}
        -
        \ddot{\mathbf{q}}^{(j)}
    \right\|_{\mathbf{W}_{\ddot q}}^2,
\end{equation}
where $j$ indexes the $N$ training samples. Throughout this section, $\mathbf{W}_{\ddot q}$, $\mathbf{W}_x$ and $\mathbf{W}_\tau$ are diagonal weighting matrices that normalise errors by the training-data standard deviations of the corresponding components.

For training on trajectory data, the forward model is paired with a numerical ODE solver. For example, RK4 integrates the predicted acceleration field to advance the state,
\begin{equation}
    \widetilde{\mathbf{x}}_{k+1}
    =
    \operatorname{RK4}\!\left(
        \dot{\mathbf{x}}
        =
        \begin{bmatrix}
            \dot{\mathbf{q}} \\
            \widetilde{\ddot{\mathbf{q}}}_\theta
        \end{bmatrix},
        \widetilde{\mathbf{x}}_k,
        \Delta t
    \right),
    \qquad
    \mathbf{x}
    =
    \begin{bmatrix}
        \mathbf{q} \\
        \dot{\mathbf{q}}
    \end{bmatrix},
\end{equation}
where $\Delta t$ is the integration time step. Training uses a multiple-shooting approach: a prediction is initialised from the measured state at every time step and rolled out over $K$ subsequent time steps, giving overlapping prediction windows. The training loss is the weighted mean-squared error over these windows,
\begin{equation}
\label{eq:traj_training}
    \mathcal{J}_{\mathrm{traj}}
    =
    \frac{1}{NK}
    \sum_{j=1}^{N}
    \sum_{k=1}^{K}
    \left\|
        \widetilde{\mathbf{x}}^{(j)}_k
        -
        \mathbf{x}^{(j)}_k
    \right\|_{\mathbf{W}_x}^2,
\end{equation}
where $j$ indexes the $N$ prediction windows and $k$ indexes the time steps within each window. Each prediction starts from its measured initial state, $\widetilde{\mathbf{x}}^{(j)}_0=\mathbf{x}^{(j)}_0$. This loss can be backpropagated either directly through the numerical solver or using the adjoint sensitivity method commonly employed in Neural ODEs~\cite{chen_neural_2018}.

For inverse-dynamics training, the loss is similarly defined as the weighted mean-squared error between the predicted and measured applied generalised forces,
\begin{equation}
\label{eq:inv_training}
    \mathcal{J}_{\mathrm{inv}}
    =
    \frac{1}{N}
    \sum_{j=1}^{N}
    \left\|
        \widetilde{\boldsymbol{\tau}}^{(j)}
        -
        \boldsymbol{\tau}^{(j)}
    \right\|_{\mathbf{W}_\tau}^2,
\end{equation}
where $j$ indexes the $N$ training samples.

Typically, the choice of training objective is dependent on the available data and its quality. A common problem in continuous-time system identification is the estimation of the derivatives $\dot{\mathbf{q}},\ddot{\mathbf{q}}$. Often the coordinate itself $\mathbf{q}$ is directly measured and $\dot{\mathbf{q}}$ and $\ddot{\mathbf{q}}$ are numerically estimated, for example by the central difference method or spectral differentiation. This typically amplifies state noise and can lead to noisy velocities and accelerations~\cite{chartrand_numerical_2011}, degrading acceleration-based and inverse-dynamics training, Eqs.~\eqref{eq:accel_training} and~\eqref{eq:inv_training}, respectively. On the other hand, trajectory training avoids the need for estimating $\ddot{\mathbf{q}}$, but suffers from missing readings of $\boldsymbol{\tau}$ between measurements that are needed for accurate forward integration. Trajectory training also  creates a more complex computational graph which increases training complexity. The training objective should be chosen based on the quality and availability of data, and the final inference task required of the model.

\section{Theoretical properties}
Having defined the dLNN architecture, we next establish the theoretical properties introduced by its dissipative parametrisation.
\label{sec:THEORY}
\subsection{Passivity, energy dissipation and stability}
\label{sec:PASS}
Consider the learned dynamical system with Lagrangian $\mathcal{L}_\theta(\mathbf{q},\dot{\mathbf{q}})$ with externally applied forces $\boldsymbol{\tau}$ and dissipative internal forces $\mathbf{Q}_\theta^{\text{diss}}(\mathbf{q},\dot{\mathbf{q}})=-\mathbf{D}_\theta(\mathbf{q},\dot{\mathbf{q}})\dot{\mathbf{q}}$ such that $\mathbf{D}_\theta(\mathbf{q},\dot{\mathbf{q}})\succ\mathbf{0}$. The system's energy is lower-bounded and defined as 

\begin{equation}
E_\theta=\dot{\mathbf{q}}^{\mathsf T}\frac{\partial \mathcal{L}_\theta}{\partial \dot{\mathbf{q}}}-\mathcal{L}_\theta=\frac{1}{2} \dot{\mathbf{q}}^{\mathsf{T}}\mathbf{M}_\theta(\mathbf{q})\dot{\mathbf{q}}+V_\theta(\mathbf{q}).
\end{equation}

\begin{proposition}[Passivity of the learned dynamics]
The learned dynamical system is passive from input $\boldsymbol{\tau}$ to output $\dot{\mathbf{q}}$ with storage function $E_\theta$.
\end{proposition}

\begin{proof}
    The mechanical energy associated with the learned Lagrangian is
    \begin{align}
        E_\theta
        &=
        \dot{\mathbf{q}}^{\mathsf T}
        \frac{\partial \mathcal{L}_\theta}{\partial \dot{\mathbf{q}}}
        - \mathcal{L}_\theta ,
        \\
        \dot{E}_\theta
        &=
        \ddot{\mathbf{q}}^{\mathsf T}
        \frac{\partial \mathcal{L}_\theta}{\partial \dot{\mathbf{q}}}
        +
        \dot{\mathbf{q}}^{\mathsf T}
        \frac{d}{dt}
        \left(
            \frac{\partial \mathcal{L}_\theta}
            {\partial \dot{\mathbf{q}}}
        \right)
        -
        \frac{\partial \mathcal{L}_\theta}{\partial \mathbf{q}}^{\mathsf T}
        \dot{\mathbf{q}}
        -
        \frac{\partial \mathcal{L}_\theta}{\partial \dot{\mathbf{q}}}^{\mathsf T}
        \ddot{\mathbf{q}},
        \\
        &=
        \dot{\mathbf{q}}^{\mathsf T}
        \left[
            \frac{d}{dt}
            \left(
                \frac{\partial \mathcal{L}_\theta}
                {\partial \dot{\mathbf{q}}}
            \right)
            -
            \frac{\partial \mathcal{L}_\theta}{\partial \mathbf{q}}
        \right],
        \\
        &=
        \dot{\mathbf{q}}^{\mathsf T}
        \left(
            \boldsymbol{\tau}
            +
            \mathbf{Q}^{\mathrm{diss}}_\theta
        \right),
        \\
        &=
        \dot{\mathbf{q}}^{\mathsf T}\boldsymbol{\tau}
        -
        \dot{\mathbf{q}}^{\mathsf T}
        \mathbf{D}_\theta(\mathbf{q},\dot{\mathbf{q}})
        \dot{\mathbf{q}},
        \\
        &<
        \dot{\mathbf{q}}^{\mathsf T}\boldsymbol{\tau},
        \qquad
        \dot{\mathbf{q}}\neq\mathbf{0},
    \end{align}
    where the final inequality follows from
    $\mathbf{D}_\theta(\mathbf{q},\dot{\mathbf{q}})\succ\mathbf{0}$.
    Hence, taking $E_\theta$ as the storage function, the learned dynamics are
    passive with input $\boldsymbol{\tau}$ and output $\dot{\mathbf{q}}$~\cite{van_der_schaft_l2-gain_2017}.
\end{proof}

\begin{corollary}
    For the unforced autonomous system, $\boldsymbol{\tau}=\mathbf{0}$, and therefore
    \begin{equation}
        \dot{E}_\theta
        =
        -\dot{\mathbf{q}}^{\mathsf T}
        \mathbf{D}_\theta(\mathbf{q},\dot{\mathbf{q}})
        \dot{\mathbf{q}}
        <0,
        \qquad
        \dot{\mathbf{q}}\neq\mathbf{0}.
    \end{equation}
Hence, the learned mechanical energy decreases strictly whenever the system is in motion. Under the usual conditions required by LaSalle's invariance principle~\cite{lasalle_extensions_1960}, trajectories approach the set of stationary equilibria satisfying $\dot{\mathbf{q}}=\mathbf{0}$ and $\nabla_{\mathbf{q}}V_\theta(\mathbf{q})=\mathbf{0}$. In engineering terms, the unforced system dissipates energy and settles towards rest. Equilibria at strict local minima of $V_\theta$ are stable, although convergence to a particular minimum is not guaranteed by energy dissipation alone.
\end{corollary}

\begin{remark}
    Notably, the dissipation inequality does not require the classical mechanical structure $\mathcal L_\theta=\frac{1}{2}\dot{\mathbf q}^{\mathsf T}\mathbf M_\theta(\mathbf q)\dot{\mathbf q}-V_\theta(\mathbf q)$. It extends to any time-independent Lagrangian $\mathcal L_\theta(\mathbf q,\dot{\mathbf q})$ for which the associated generalised energy is a valid storage function. Consequently, the proposed dissipative parametrisation is applicable beyond classical mechanical Lagrangians.
\end{remark}

\subsection{Representation of strictly dissipative forces}
\label{sec:express}

Consider a generalised dissipative force $\mathbf{Q}(\mathbf{q},\dot{\mathbf{q}})$ satisfying $\mathbf{Q}(\mathbf{q},\mathbf{0})=\mathbf{0}$ and
\begin{equation}
    \dot{\mathbf{q}}^{\mathsf T}\mathbf{Q}(\mathbf{q},\dot{\mathbf{q}}) < 0,
    \qquad \dot{\mathbf{q}}\neq\mathbf{0}.
\end{equation}

\begin{proposition}[Representation of strictly dissipative forces]
For any $\mathbf{Q}$ satisfying the preceding conditions, there exists a symmetric positive-definite matrix $\mathbf{D}\succ\mathbf{0}$ such that $\mathbf{Q}=-\mathbf{D}(\mathbf{q},\dot{\mathbf{q}})\dot{\mathbf{q}}$.
\end{proposition}

\begin{proof}
For $\dot{\mathbf{q}}\neq\mathbf{0}$, define
\begin{equation}
\label{eq:D_form_proof}
    \mathbf{D}=-\frac{\mathbf{Q}\mathbf{Q}^{\mathsf T}}{\dot{\mathbf{q}}^{\mathsf T}\mathbf{Q}}
    +\varepsilon\left(\mathbf{I}-\frac{\dot{\mathbf{q}}\dot{\mathbf{q}}^{\mathsf T}}{\|\dot{\mathbf{q}}\|^2}\right),
    \qquad \varepsilon>0,
\end{equation}
where the second term is a positive multiple of the orthogonal
projection onto the subspace perpendicular to $\dot{\mathbf{q}}$.
At $\dot{\mathbf{q}}=\mathbf{0}$, any symmetric positive-definite
$\mathbf{D}(\mathbf{q},\mathbf{0})$ satisfies
$\mathbf{Q}=-\mathbf{D}\dot{\mathbf{q}}$, since
$\mathbf{Q}(\mathbf{q},\mathbf{0})=\mathbf{0}$.
For $\dot{\mathbf{q}}\neq\mathbf{0}$, multiplying
Eq.~\eqref{eq:D_form_proof} by $-\dot{\mathbf{q}}$ recovers
\begin{align}
    -\mathbf{D}\dot{\mathbf{q}}
    &=\frac{\mathbf{Q}\mathbf{Q}^{\mathsf T}\dot{\mathbf{q}}}{\dot{\mathbf{q}}^{\mathsf T}\mathbf{Q}}
    -\varepsilon\left(\mathbf{I}-\frac{\dot{\mathbf{q}}\dot{\mathbf{q}}^{\mathsf T}}{\|\dot{\mathbf{q}}\|^2}\right)\dot{\mathbf{q}} \\
    &=\mathbf{Q}.
\end{align}
It remains to show that $\mathbf{D}$ is positive definite. For any $\mathbf{x}\neq\mathbf{0}$,

\begin{align}
    \mathbf{x}^{\mathsf T}\mathbf{D}\mathbf{x}
    &=
    -\frac{\mathbf{x}^{\mathsf T}\mathbf{Q}\mathbf{Q}^{\mathsf T}\mathbf{x}}
    {\dot{\mathbf{q}}^{\mathsf T}\mathbf{Q}}
    +\varepsilon\left(
        \mathbf{x}^{\mathsf T}\mathbf{x}
        -\frac{\mathbf{x}^{\mathsf T}\dot{\mathbf{q}}\dot{\mathbf{q}}^{\mathsf T}\mathbf{x}}
        {\|\dot{\mathbf{q}}\|^2}
    \right) \\
    &=
    -\frac{(\mathbf{Q}^{\mathsf T}\mathbf{x})^2}
    {\dot{\mathbf{q}}^{\mathsf T}\mathbf{Q}}
    +\varepsilon\left(
        \|\mathbf{x}\|^2
        -\frac{(\dot{\mathbf{q}}^{\mathsf T}\mathbf{x})^2}
        {\|\dot{\mathbf{q}}\|^2}
    \right).
\end{align}
The first term is non-negative since $\dot{\mathbf{q}}^{\mathsf T}\mathbf{Q}<0$. The second term is also non-negative by the Cauchy--Schwarz inequality, $(\dot{\mathbf{q}}^{\mathsf T}\mathbf{x})^2\leq\|\dot{\mathbf{q}}\|^2\|\mathbf{x}\|^2$, and vanishes only when $\mathbf{x}$ is parallel to $\dot{\mathbf{q}}$. For $\mathbf{x}=c\dot{\mathbf{q}}\neq\mathbf{0}$, the first term becomes
\begin{equation}
    -c^2\dot{\mathbf{q}}^{\mathsf T}\mathbf{Q}>0.
\end{equation}
Hence $\mathbf{x}^{\mathsf T}\mathbf{D}\mathbf{x}>0$ for every $\mathbf{x}\neq\mathbf{0}$, and therefore $\mathbf{D}\succ\mathbf{0}$.
\end{proof}

\begin{remark}
    This argument simply proves the existence of $\mathbf{D}(\mathbf{q},\dot{\mathbf{q}})$ for $\mathbf{Q}(\mathbf{q},\dot{\mathbf{q}})$. $\mathbf{D}(\mathbf{q},\dot{\mathbf{q}})$ is not usually unique nor limited to the form in Eq.~\eqref{eq:D_form_proof} for a given $\mathbf{Q}(\mathbf{q},\dot{\mathbf{q}})$. This argument also makes no assumption on the continuity of $\mathbf{D}(\mathbf{q},\dot{\mathbf{q}})$ or $\mathbf{Q}(\mathbf{q},\dot{\mathbf{q}})$.
\end{remark}

\subsection{Conservative--dissipative separation}
\label{sec:SEPER}

Consider a mechanical system with the inverse-dynamics relationship
\begin{equation}
    \boldsymbol{\tau}
    =
    \mathbf{M}(\mathbf{q})\ddot{\mathbf{q}}
    +
    \mathbf{c}(\mathbf{q},\dot{\mathbf{q}})
    +
    \nabla_{\mathbf{q}}V(\mathbf{q})
    -
    \mathbf{Q}(\mathbf{q},\dot{\mathbf{q}}),
\end{equation}
where $\mathbf{c}(\mathbf{q},\dot{\mathbf{q}})$ is determined by
$\mathbf{M}(\mathbf{q})$ as in Eq.~\eqref{eq:corio}.

\begin{proposition}[Uniqueness of the conservative--dissipative decomposition]
If two decompositions produce the same applied forces
$\boldsymbol{\tau}$ for all configurations, velocities, and
accelerations, and both satisfy
$\mathbf{Q}(\mathbf{q},\mathbf{0})=\mathbf{0}$, then their mass
matrices and dissipative forces are identical, and their potential
energies differ only by an additive constant.
\end{proposition}

\begin{proof}
Suppose that a second set $\mathbf{M}'(\mathbf{q})$, $V'(\mathbf{q})$, and $\mathbf{Q}'(\mathbf{q},\dot{\mathbf{q}})$ describes exactly the same inverse dynamics. Since $\mathbf{M}(\mathbf{q})$ alone determines the dependence of $\boldsymbol{\tau}$ on $\ddot{\mathbf{q}}$, equality for all $\ddot{\mathbf{q}}$ requires
\begin{equation}
    \mathbf{M}'(\mathbf{q})=\mathbf{M}(\mathbf{q}).
\end{equation}
The corresponding term $\mathbf{c}(\mathbf{q},\dot{\mathbf{q}})$ is therefore also unchanged. Let
\begin{equation}
    V' = V+\Delta V,
    \qquad
    \mathbf{Q}'=\mathbf{Q}+\Delta\mathbf{Q}.
\end{equation}
For the two decompositions to produce the same dynamics,
\begin{equation}
    \nabla_{\mathbf{q}}\Delta V(\mathbf{q})
    =
    \Delta\mathbf{Q}(\mathbf{q},\dot{\mathbf{q}})
\end{equation}
for all $\mathbf{q}$ and $\dot{\mathbf{q}}$. Evaluating this at $\dot{\mathbf{q}}=\mathbf{0}$ and using $\mathbf{Q}(\mathbf{q},\mathbf{0})=\mathbf{Q}'(\mathbf{q},\mathbf{0})=\mathbf{0}$ gives
\begin{equation}
    \nabla_{\mathbf{q}}\Delta V(\mathbf{q})=\mathbf{0}.
\end{equation}
Hence $\Delta V$ can only be an additive constant. It then follows that $\Delta\mathbf{Q}=\mathbf{0}$ for all $\mathbf{q}$ and $\dot{\mathbf{q}}$. Therefore $\mathbf{M}$ and $\mathbf{Q}$ are unique, while $V$ is unique up to an arbitrary additive constant.
\end{proof}

\begin{remark}
    It should be noted that arbitrary black-box generalised forces $\mathbf{Q}_\theta(\mathbf{q},\dot{\mathbf{q}})$ or black-box Rayleigh dissipation functions $\mathcal{R}_\theta(\mathbf{q},\dot{\mathbf{q}})$ do not have this uniqueness property and the dissipation term can readily absorb conservative dynamics, especially from $V_\theta$.
\end{remark}

\begin{remark}
This result establishes uniqueness when the inverse-dynamics relationship is known exactly for all configurations, velocities, and accelerations. It does not guarantee that training on finite, noisy data will recover the true physical components, since limited data coverage, model approximation, and optimisation errors may
prevent exact identification.
\end{remark}

\begin{remark}
The same uniqueness argument applies to forced forward dynamics when the acceleration response is known for all states and applied forces. For unforced dynamics, however, multiplying $\mathbf{M}$, $V$, and $\mathbf{Q}$ by the same positive constant leaves the predicted motion unchanged, as does adding a constant to $V$. The decomposition therefore cannot be uniquely identified from unforced motion alone; additional ambiguities may also exist
depending on the system.
\end{remark}

\section{Validation with synthetic and real-world data}
To validate this architecture, the model is evaluated on two mechanical systems of varying complexity. The first is the chaotic double pendulum. This system exhibits fine sensitivity to model accuracy where very small modelling errors can quickly propagate to the full divergence of predicted trajectories~\cite{shinbrot_chaos_1992}. This property makes it a particularly attractive benchmark for model comparison and allows for a clean separation for even very subtle disparities in modelling performance. The models are initially validated using simulated data and subsequently validated using experimental data.

The second more complex problem concerns the inverse modelling of a real six-degree-of-freedom industrial robotic arm~\cite{weigand_dataset_2022}. For this problem, there is significantly more complexity in both the conservative and dissipative components of the true dynamics. The frictional low-speed behaviour is notably complex and the conservative dynamics concern a real, three-dimensional, forced mechanical system with six degrees of freedom.

\newcolumntype{Y}{>{\centering\arraybackslash}X}
\begin{table}[htbp]
    \centering
    \small
    
    \caption{Summary of the three benchmark datasets and associated training and evaluation tasks. Double-pendulum dataset sizes are reported as sampling rate $\times$ trajectory duration $\times$ number of trajectories; robot dataset sizes are reported as sampling rate $\times$ total duration.}
    \label{tab:benchmarks}
    
    \begin{tblr}{
        width=\textwidth,
        colspec={X[l] Q[c] X[c] X[c] X[c] X[c]},
        rows={m},
        hlines,
        vline{1,2,7}={solid},
    }
    \textbf{System} &
    \textbf{DoF} &
    \textbf{Training task} &
    \textbf{Evaluation task} &
    \textbf{Training data} &
    \textbf{Test data} \\
    
    Simulated double pendulum &
    2 &
    $(\mathbf{q}_t,\dot{\mathbf{q}}_t)\rightarrow\ddot{\mathbf{q}}_t$ &
    $(\mathbf{q}_0,\dot{\mathbf{q}}_0)\rightarrow
    (\mathbf{q}_{1:T},\dot{\mathbf{q}}_{1:T})$ &
    $90\,\mathrm{Hz}\times11.7\,\mathrm{s}\times445$ &
    $90\,\mathrm{Hz}\times5\,\mathrm{s}\times350$ \\
    
    Experimental double pendulum &
    2 &
    $(\mathbf{q}_t,\dot{\mathbf{q}}_t)\rightarrow
    (\mathbf{q}_{t+1},\dot{\mathbf{q}}_{t+1})$ &
    $(\mathbf{q}_0,\dot{\mathbf{q}}_0)\rightarrow
    (\mathbf{q}_{1:T},\dot{\mathbf{q}}_{1:T})$ &
    $90\,\mathrm{Hz}\times11.7\,\mathrm{s}\times5$ &
    $90\,\mathrm{Hz}\times5\,\mathrm{s}\times350$ \\
    
    Industrial robot arm &
    6 &
    $(\mathbf{q}_t,\dot{\mathbf{q}}_t,\ddot{\mathbf{q}}_t)
    \rightarrow\boldsymbol{\tau}_t$ &
    $(\mathbf{q}_t,\dot{\mathbf{q}}_t,\ddot{\mathbf{q}}_t)
    \rightarrow\boldsymbol{\tau}_t$ &
    $10\,\mathrm{Hz}\times4020\,\mathrm{s}$ &
    $10\,\mathrm{Hz}\times360\,\mathrm{s}$ \\
    \end{tblr}

\end{table}

Across these benchmarks (Table~\ref{tab:benchmarks}), the model is validated on both toy and industrial problems; real and simulated data; forced and unforced systems; and through training on acceleration prediction (Eq.\eqref{eq:accel_training}), trajectory prediction (Eq.\eqref{eq:traj_training}), and torque prediction (Eq.\eqref{eq:inv_training}). The double pendulum dataset and code implementations for both the double pendulum and robot arm are available on \url{https://github.com/robseabourne/Dissipative_Lagrangian_Neural_Networks}.

\begin{figure}[h]
    \centering

    \begin{minipage}[c]{0.44\linewidth}
        \centering
        \includegraphics[width=\linewidth]{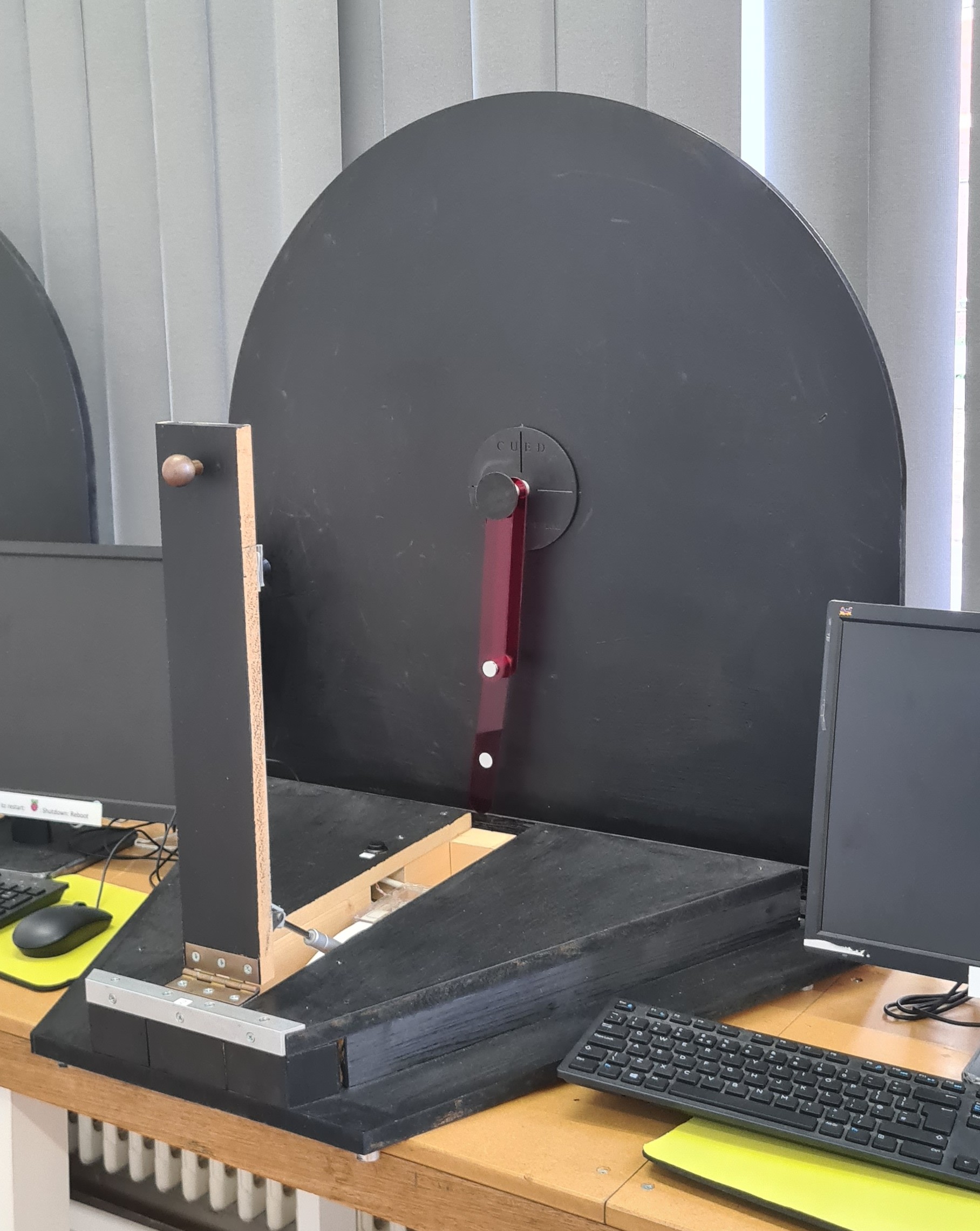}
    \end{minipage}
    \hfill
    \begin{minipage}[c]{0.54\linewidth}
        \centering
        \includegraphics[width=0.9\linewidth]{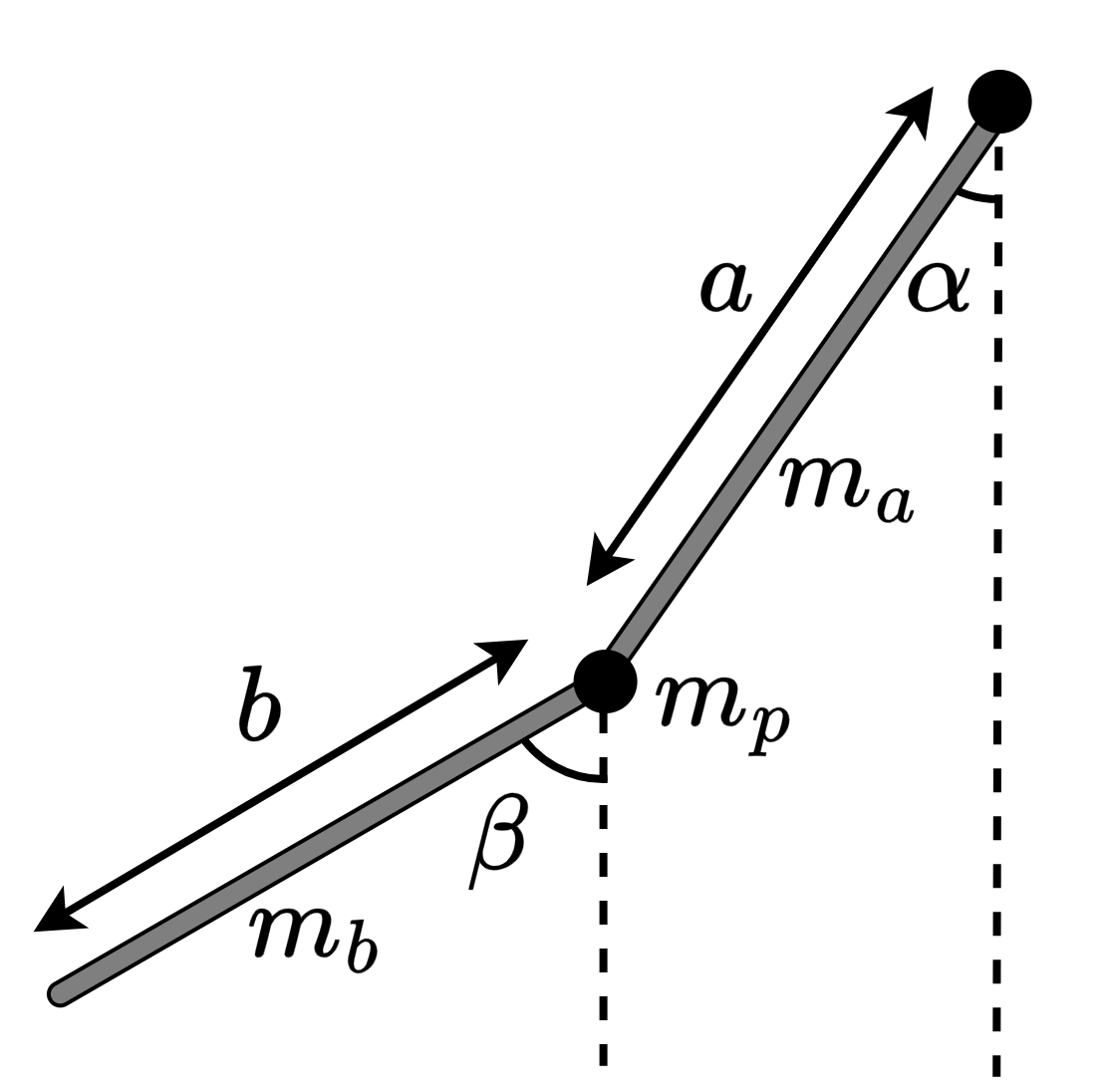}
    \end{minipage}

    \caption{Double-pendulum benchmark system. (\textbf{Left}) Experimental double-pendulum apparatus used to collect the real-world trajectory data. (\textbf{Right}) Corresponding mechanical model, showing the absolute angles $\alpha$ and $\beta$, rod lengths $a$ and $b$, rod masses $m_a$ and $m_b$, and intermediate point mass $m_p$.}
    \label{fig:APPARATUS}
\end{figure}
\subsection{Double-pendulum benchmarks}
The first benchmark problem is a simulated chaotic double pendulum, providing the most idealised conditions for the models, with no measurement noise, external disturbances, or temporal variation in the system dynamics. Nominal system parameters are identified to match the dynamics of a real physical double pendulum within our lab. The analytical model (Eqs.~\eqref{eq:double_pendulum_alpha} and~\eqref{eq:double_pendulum_beta}) is simulated for friction parameters that are 0\%, 5\%, 20\%, 100\% and 200\% of the nominal $T_1,T_2,b_1,b_2$. This allows us to test the relative performance of the models under increasing amounts of dissipation, bridging the gap between perfectly conservative systems and realistic levels of dissipation.

The training dataset contains 445 simulated trajectories of 11.7 s at a sampling rate of 90 Hz. These trajectories are split into individual time steps and the model is trained on batches of input-output pairs $\mathbf{q}_t,\dot{\mathbf{q}}_t \rightarrow \ddot{\mathbf{q}}_t$. 

The test set contains 50 trajectories of 11.7 s  which are split into 350 overlapping trajectories of 5 s, beginning at 0 s, 1 s ... 6 s. For each test trajectory, a model will predict an entire test trajectory of 5 s from the initial conditions $\mathbf{q}_0,\dot{\mathbf{q}}_0$. The predictions are then evaluated for both pointwise prediction accuracy and physical realism. The first is evaluated by the time $t_\mathrm{div}$ at which the predicted lower-arm angle $\beta$ first differs from the corresponding measured angle by $180^\circ$ (with unwrapped angles), i.e., when the prediction has completely diverged from reality. If trajectories do not diverge, $t_\mathrm{div}$ is counted as the full 5 s. The latter is evaluated using the mean-squared error between the energies of the measured and predicted trajectories, calculated using the analytical energy function $E(\mathbf{q},\dot{\mathbf{q}})$ defined in Appendix \ref{sec:DP-EOM}, Eq.~\eqref{eq:ENERGY},
\begin{equation}
    \text{Energy MSE}
    =
    \frac{1}{NK}
    \sum_{j=1}^{N}
    \sum_{k=1}^{K}
    \left(
        E(\mathbf{q}^{(j)}_k,\dot{\mathbf{q}}^{(j)}_k)
        -
        E(\widetilde{\mathbf{q}}^{(j)}_k,
          \widetilde{\dot{\mathbf{q}}}^{(j)}_k)
    \right)^2,
\end{equation}
where $j$ indexes the $N$ test trajectories and $k$ indexes the $K$ time steps within each trajectory.

It is imperative to also test the model on imperfect real-world data to evaluate the robustness of the dLNN's performance. To achieve this, a dataset is collected from a real double pendulum. In this apparatus (Figure~\ref{fig:APPARATUS}), reflective markers on the two pendulum arms are tracked by an IR camera at 90 Hz. A computer vision algorithm calculates $\alpha$ and $\beta$ from the marker positions, and $\dot{\alpha}$ and $\dot{\beta}$ are estimated using the central difference method. To evaluate the system performance in the low-data regime, this time the training set only consists of five trajectories of 11.7 s. The test set similarly contains 350 trajectories of 5 s each.

Trajectory-based training is used for the experimental data because accurate acceleration measurements are unavailable. The learned acceleration field is integrated using a fixed-step fourth-order Runge–Kutta (RK4) solver, with a step size equal to the data sampling interval. Training uses a one-step prediction horizon, with overlapping state pairs $(\mathbf{q}_t,\dot{\mathbf{q}}_t)\rightarrow(\mathbf{q}_{t+1},\dot{\mathbf{q}}_{t+1})$.

To evaluate the performance of the models, the same metrics of divergence time $t_{\mathrm{div}}$ and energy MSE are used. The simulation parameters are chosen to match the experimental double pendulum, allowing the same analytical energy expression, Eq.~\eqref{eq:ENERGY}, to be used to calculate the energy MSE for the real system. State estimation errors and uncertainty in the identified parameters affect the accuracy of the calculated energy. To reduce noise, both the measured and predicted energy histories are smoothed using a Gaussian filter with $\sigma=0.11$ s before calculating the MSE.

\subsection{Robot inverse-dynamics benchmark}
To evaluate the ability of the dLNN architecture to model more complex inverse dynamics, the Industrial Robot Identification Benchmark dataset developed by Weigand et al. is used~\cite{weigand_dataset_2022}. The dataset comprises 73 minutes of measured trajectory and torque data from a six-degree-of-freedom KUKA KR300 R2500 ultra SE industrial robot arm, processed and resampled at 10 Hz. The authors of this benchmark claim \textit{"A major challenge of the benchmark are nonlinear effects which occur during standstill or low-velocity phase}", highlighting the complex low-velocity frictional behaviour of the system, making it an interesting testing ground for the dLNN model. Model performance is evaluated by calculating the $R^2$ score for the predicted torques of each joint over the test set and then averaging these scores across the six joints, as described in detail in~\cite{weigand_dataset_2022}.

Owing to this complex low-velocity friction behaviour, the dLNN model uses the Coulomb-style augmentation introduced in Section~\ref{sec:COULOMB EXTRA} for this benchmark, with smoothing parameter $k=10^6$. An equivalent augmentation is not used for the double-pendulum benchmarks, where it was found to provide no meaningful improvement.

\subsection{Baselines}
Alongside the dLNN, some standard baseline models are used for comparison. The first is the Generalised Lagrangian Neural Network or gLNN~\cite{xiao_generalized_2025}. For this implementation, the conservative dynamics are learned in the same structured way as the dLNN but the dissipative forces $\mathbf{Q}_\theta(\mathbf{q},\dot{\mathbf{q}})$ are parametrised directly as the output of an unconstrained neural network. The second baseline model is the standard Lagrangian Neural Network or LNN~\cite{cranmer_lagrangian_2020,lutter_deep_2019}. In this implementation, this model is identical in structure to the conservative part of the dLNN as described in Section~\ref{sec:Learning Conservative Mechanics} with no generalised forces learned. 

The final baseline model is an unconstrained black-box neural network. For the forward-dynamics task, this directly maps $(\mathbf{q}_t,\dot{\mathbf{q}}_t,(\boldsymbol{\tau}_t) \rightarrow \ddot{\mathbf{q}}_t)$ and is referred to as a Neural Ordinary Differential Equation or nODE~\cite{chen_neural_2018} to match existing literature. For the inverse-dynamics task, a black-box model directly learns $(\mathbf{q}_t,\dot{\mathbf{q}}_t,\ddot{\mathbf{q}}_t \rightarrow \boldsymbol{\tau}_t)$ and is referred to simply as a Multi-Layer-Perceptron or MLP, again to match existing inverse robotics literature nomenclature.
Additionally, for the inverse robotics benchmark, a linear baseline is provided that has been identified using traditional system identification methods~\cite{weigand_dataset_2022}.

\subsection{Training details and hyperparameters}

\begin{table}[htbp]
\centering
\small
\setlength{\tabcolsep}{6pt}
\renewcommand{\arraystretch}{1.3}
\caption{Neural-network architectures and training parameters used for the double-pendulum and industrial robot benchmarks.}
\label{tab:training_parameters}
\begin{tabular}{|l|cc|}
\hline
Parameter & Double pendulum & Robot \\
\hline
Primary network size (width $\times$ depth) & $128\times4$ & $256\times4$ \\
gLNN/dLNN force network size (width $\times$ depth) & $64\times3$ & $64\times3$ \\
Learning rates, $\mathrm{LR}1 \rightarrow \mathrm{LR}2$ & $10^{-2}\rightarrow10^{-4}$ & $10^{-2}\rightarrow10^{-3}$ \\
Training time & $10\,\mathrm{min}$ & $2\,\mathrm{min}$ \\
Mass-matrix offset, $\epsilon$ & $10^{-2}$ & $10^{-1}$ \\
Regularisation & None & Dropout $p=0.4$ \\
\hline
\end{tabular}
\end{table}
To establish the optimal architecture dimensions for each model, a TPE Bayesian Hyper-parameter optimisation was completed using the Python package Optuna~\cite{akiba_optuna_2019} to optimise over layer number, layer width and learning rates for both benchmarks, using a validation set partitioned from the training data. For all model architectures, near-optimal performance was achieved using $128\times4$ networks for the double-pendulum benchmark and $256\times4$ networks for the robot benchmark, for both the black-box models (nODE/MLP) and the conservative branches of the LNN, gLNN, and dLNN. The non-conservative branches of the gLNN and dLNN used $64\times3$ networks for both benchmarks. It should be noted that model performance was observed to be robust to even significant changes in these hyper-parameters. The activation function GELU was found to be optimal. To encode periodicity, all angle inputs are expanded through $\sin$ and $\cos$ activations before entering the networks.

To optimise the model parameters, a typical Adam optimiser~\cite{kingma_adam_2015} was used alongside an exponential learning rate scheduler, sweeping between two fixed rates, LR1 and LR2. Batch sizes were chosen to maximise GPU utilisation within available memory constraints, with values of approximately 100,000 found to give the best performance. Overall, the choices of model dimensions, batch size, and learning rate were found to impact the training rate but did not noticeably affect the final performance on the validation set.

A common problem with LNN-style architectures for forward-dynamics problems (Eq.~\eqref{eq:forward_dyn}) is the necessity to invert a matrix in training. This can cause exploding gradients and training instability if ill-conditioned mass matrices are encountered in training. To solve this, a small diagonal offset $\epsilon \mathbf{I}$ is added to the learned mass matrix in Eq.~\eqref{eq:MASS_MATRIX}, and gradient clipping~\cite{hamzaogullari_learning_2026} is implemented. These are found to solve this issue completely. For the gLNN and dLNN architectures, it was also found to be beneficial to scale the output of the non-conservative network by a factor of 0.01, which reflected the expected scale of the dissipation term.

The benefit of embedding structure into the models can be thought of as a strong form of regularisation. As such, for the double pendulum benchmark, no explicit regularisation is used. However, for the robot benchmark, the availability of data compared to the model dimension was significantly reduced so the risk of overfitting was more serious. For this benchmark, weight decay~\cite{loshchilov_decoupled_2017} was observed to be ineffective but dropout~\cite{srivastava_dropout_nodate} was found to be extremely strong for reducing overfitting and significantly improved final validation set performance.

Every model was trained within a Google Colab environment using an NVIDIA A100 GPU. The models were trained for 2 minutes for the robot task and 10 minutes for both double pendulum tasks. This was more than enough time for all models to plateau to their optimal validation set performances. Every experiment was repeated three times with different model seeds, and results were simply averaged.

\section{Results}

\subsection{Simulated double pendulum}

\begin{figure}[h]
    \centering

    \begin{minipage}[c]{0.48\linewidth}
        \centering
        \includegraphics[width=\linewidth]{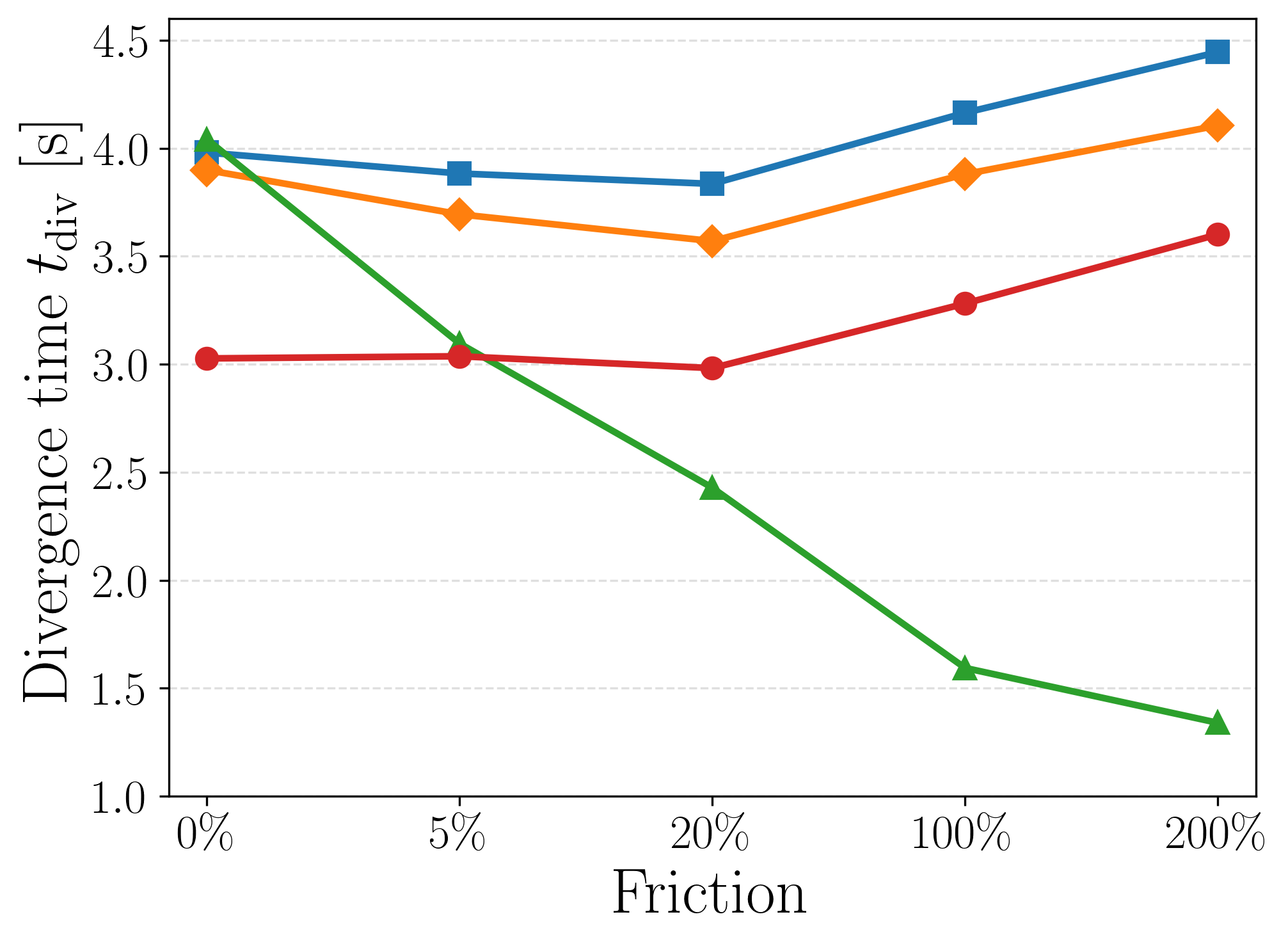}
    \end{minipage}
    \hfill
    \begin{minipage}[c]{0.48\linewidth}
        \centering
        \includegraphics[width=\linewidth]{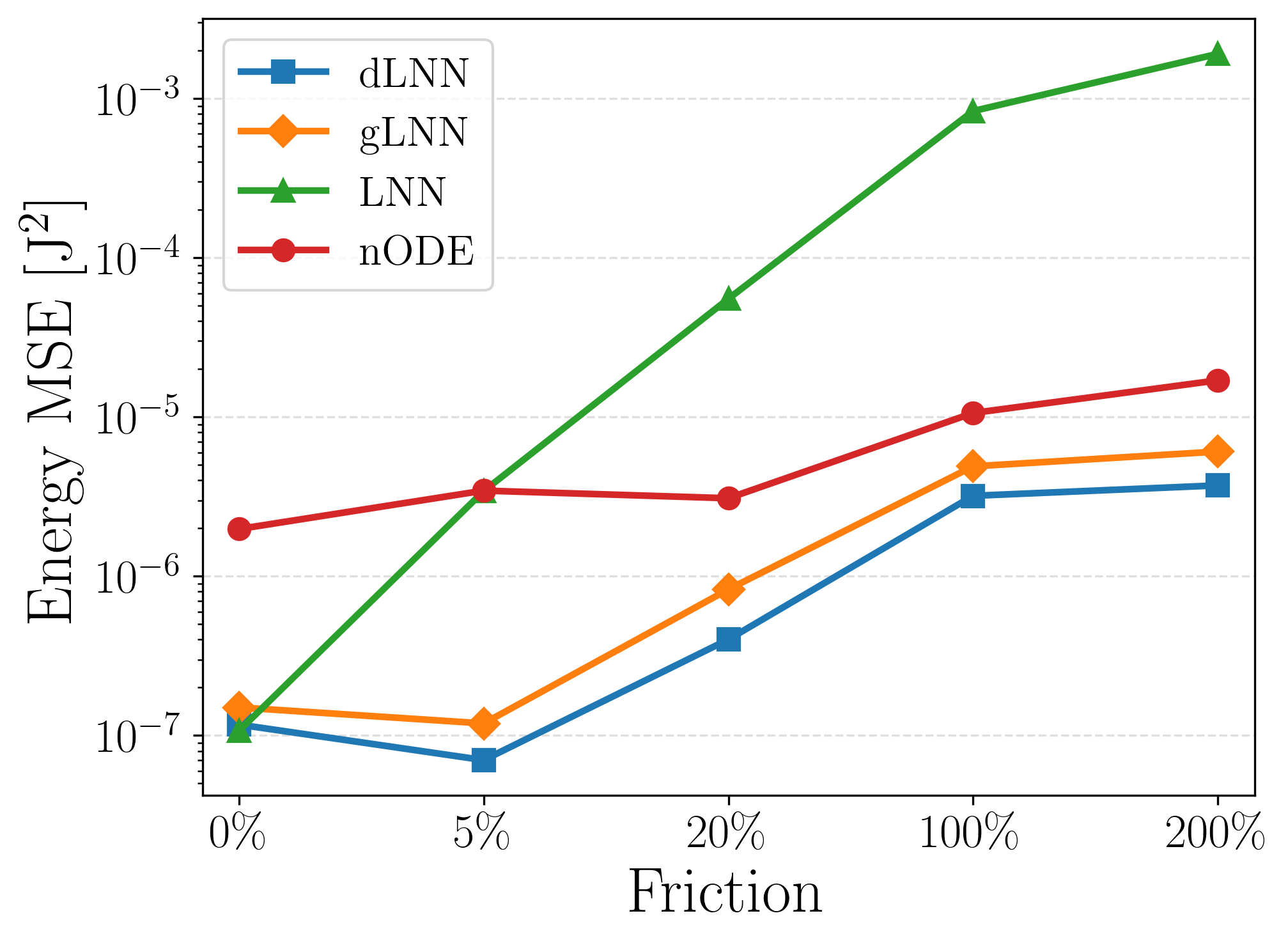}
    \end{minipage}

    \caption{Trajectory-prediction performance on the simulated double-pendulum benchmark across increasing levels of friction. (\textbf{Left}) Mean divergence time $t_{\mathrm{div}}$. (\textbf{Right}) Mean energy mean-squared error (MSE). Friction is expressed as a percentage of the nominal experimental friction parameters.}
    \label{fig:sim_results}
\end{figure}

Figure~\ref{fig:sim_results} shows the trajectory-prediction performance across the simulated friction sweep. A consistent performance hierarchy is observed across all non-zero friction levels and both evaluation metrics, with the dLNN achieving the highest performance, followed by the gLNN and nODE. At the nominal friction level ($100\%$), the dLNN achieves an average divergence time of $4.16\,\mathrm{s}$ and an energy MSE of $3.2\times10^{-6}\,\mathrm{J}^2$, compared with $3.88\,\mathrm{s}$ and $4.9\times10^{-6}\,\mathrm{J}^2$ for the gLNN, and $3.28\,\mathrm{s}$ and $1.1\times10^{-5}\,\mathrm{J}^2$ for the nODE.

The LNN performs comparably to the dLNN and gLNN in the absence of friction, but its performance decreases rapidly as friction is introduced. This degradation is already apparent at $5\%$ of the nominal friction magnitude and becomes progressively larger with increasing friction. At nominal friction, the LNN achieves a divergence time of $1.59\,\mathrm{s}$ and an energy MSE of $8.4\times10^{-4}\,\mathrm{J}^2$, approximately $260\times$ greater than that of the dLNN.

\subsection{Experimental double pendulum}
\begin{figure}[h]
    \centering

    \begin{minipage}[c]{0.48\linewidth}
        \centering
        \includegraphics[width=\linewidth]{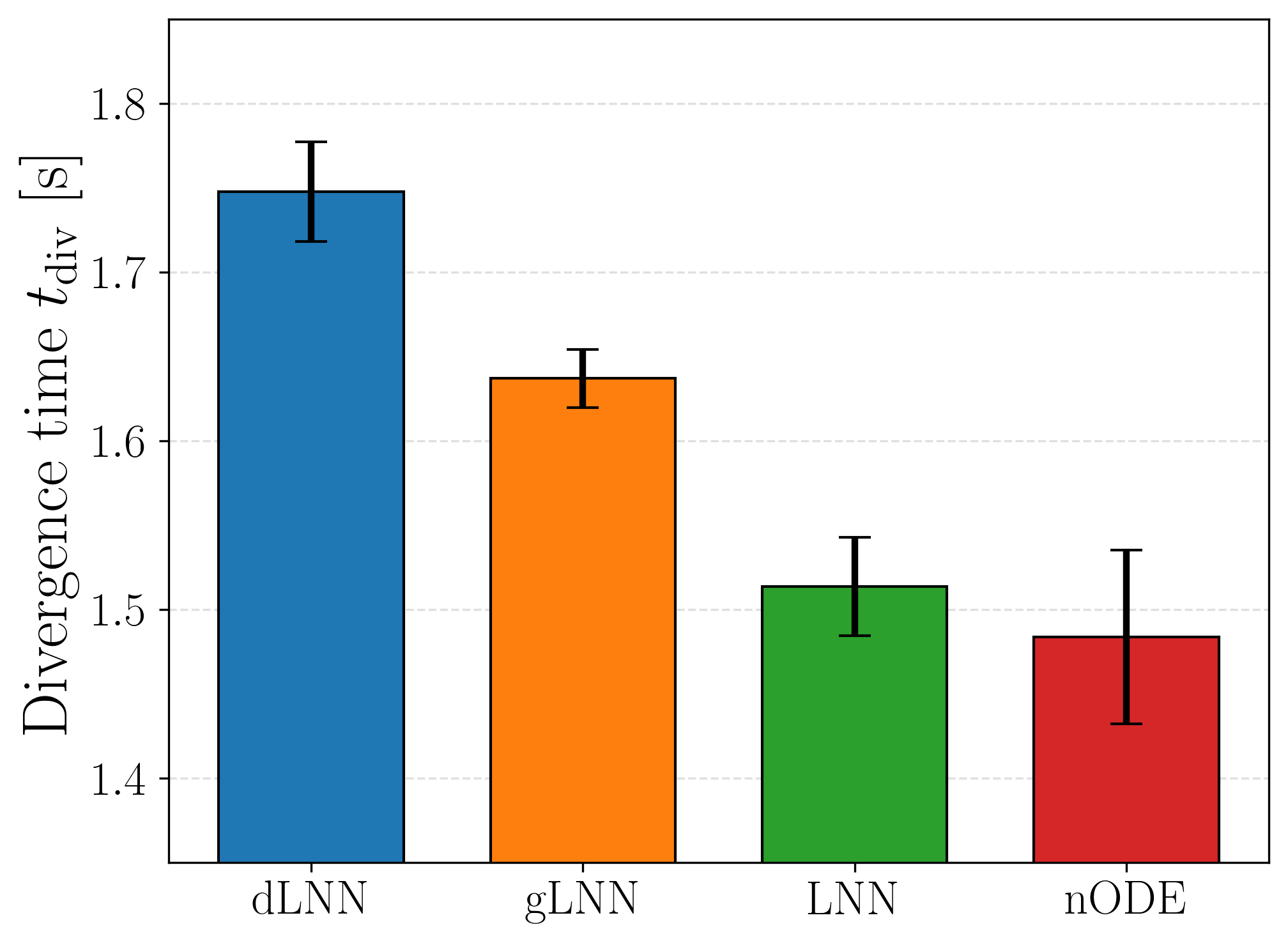}
    \end{minipage}
    \hfill
    \begin{minipage}[c]{0.48\linewidth}
        \centering
        \includegraphics[width=\linewidth]{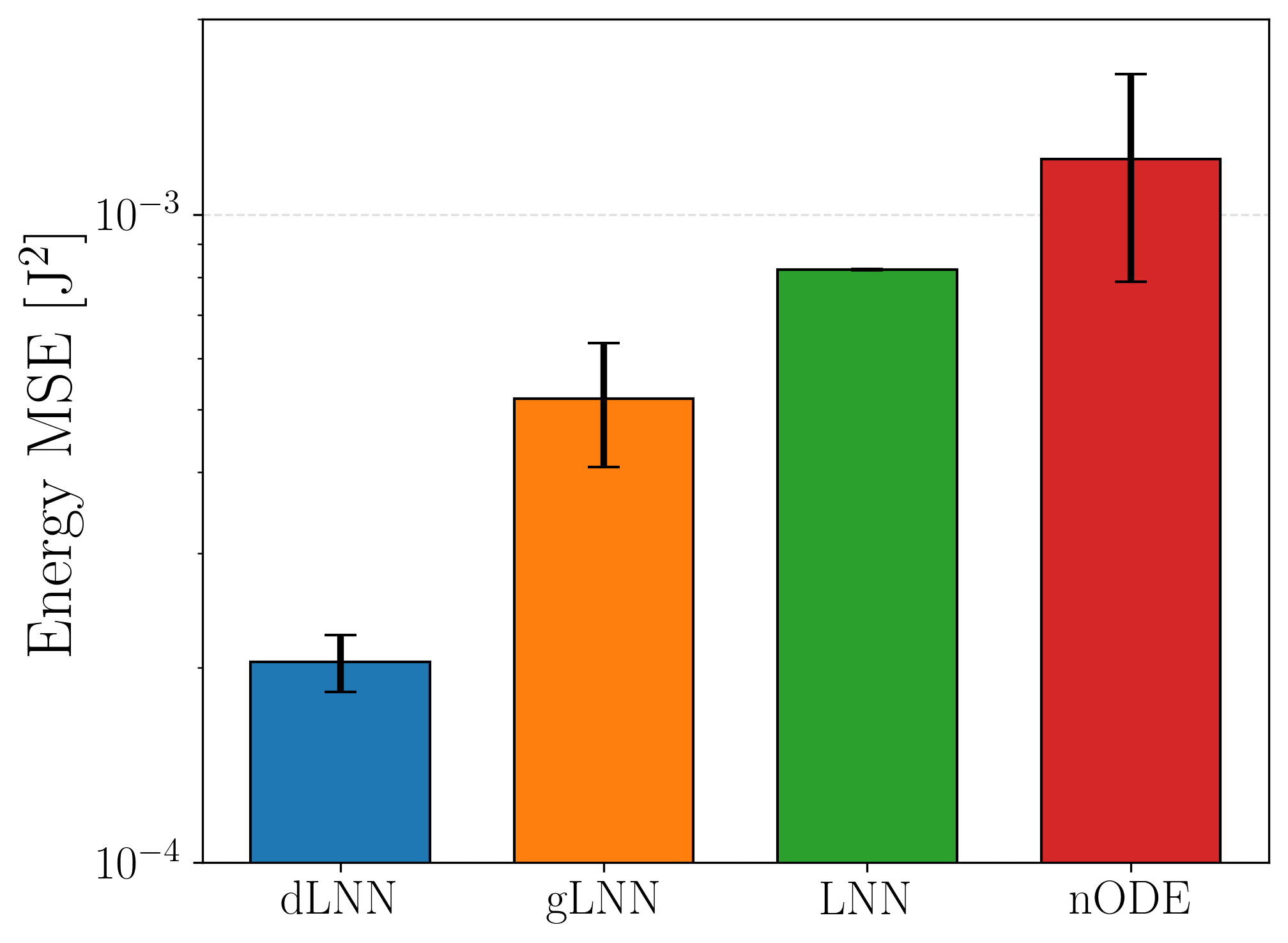}
    \end{minipage}

    \caption{Trajectory-prediction performance on the reduced experimental double-pendulum benchmark. (\textbf{Left}) Mean divergence time $t_{\mathrm{div}}$. (\textbf{Right}) Mean energy mean-squared error (MSE). Error bars show performance range over three random initialisations.}
    \label{fig:exp_results}
\end{figure}

Figure~\ref{fig:exp_results} shows the trajectory-prediction performance on the reduced experimental double-pendulum dataset. The dLNN achieves the highest performance in both evaluation metrics, followed by the gLNN. The dLNN and gLNN achieve average divergence times of $1.75\,\mathrm{s}$ and $1.64\,\mathrm{s}$, respectively, with the dLNN achieving an energy MSE approximately 40\% of that of the gLNN. The LNN and nODE show lower and similar trajectory-prediction performance, with average divergence times of $1.51\,\mathrm{s}$ and $1.48\,\mathrm{s}$, respectively.

In addition to the aggregate performance metrics, a small number of nODE predictions exhibit non-physical unstable behaviour. An example is shown in Figure~\ref{fig:unstable_prediction} (alongside a representative dLNN prediction). This instability is never observed in the dLNN predictions.

\begin{figure}[h]
    \centering

    \begin{minipage}[c]{\linewidth}
        \centering
        \includegraphics[width=\linewidth]{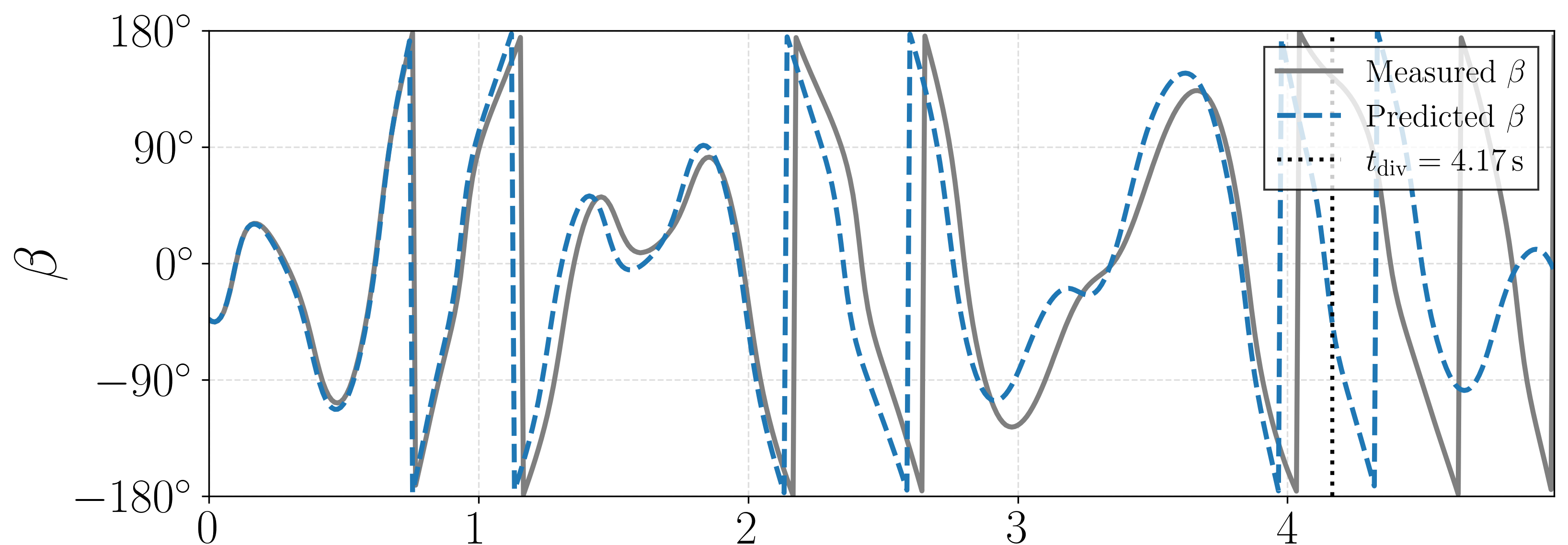}
    \end{minipage}

    \vspace{0.2cm}

    \begin{minipage}[c]{\linewidth}
        \centering
        \includegraphics[width=\linewidth]{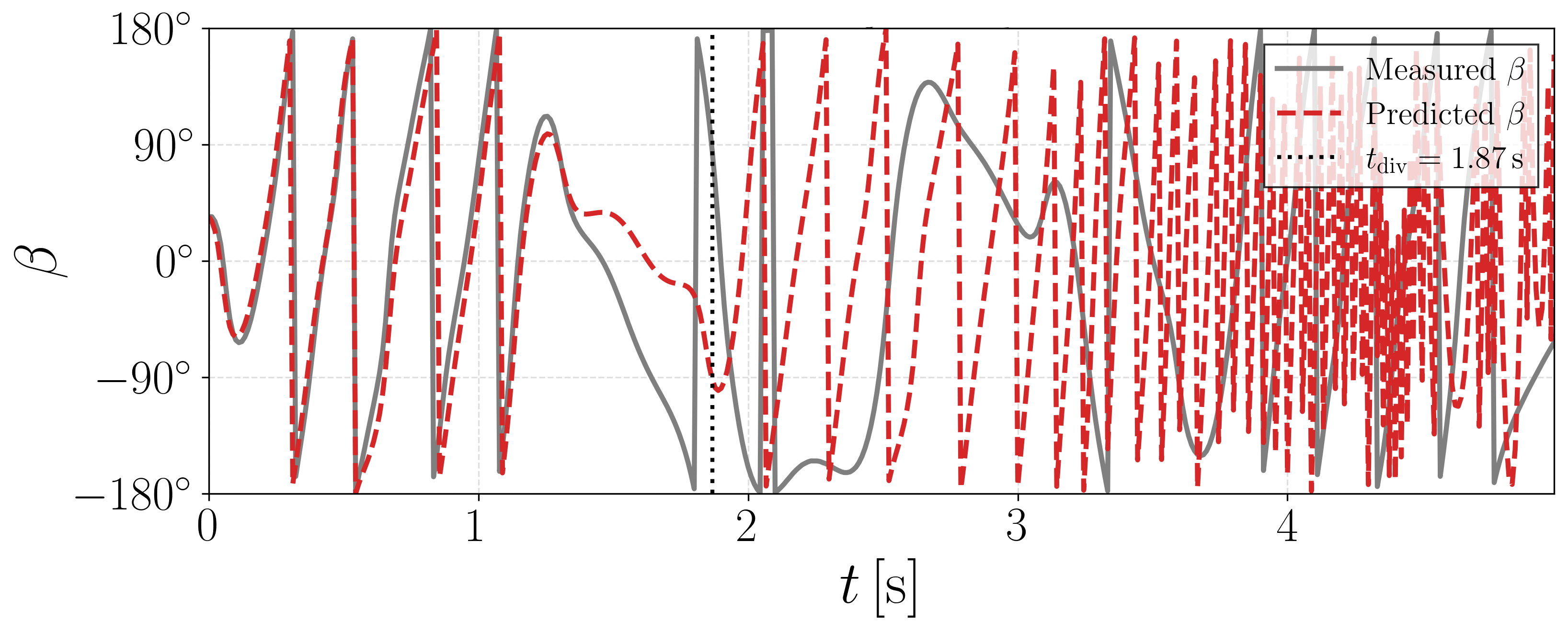}
    \end{minipage}

    \caption{Example trajectory predictions from the experimental double-pendulum benchmark, illustrating the predictive performance of the dLNN and the non-physical instability observed in a small subset of nODE predictions. (\textbf{Top}) Measured and dLNN-predicted lower-arm angle $\beta$ for trajectory 166, with initial state $(\alpha_0,\dot{\alpha}_0,\beta_0,\dot{\beta}_0)=(43.7^\circ,-446.8^\circ\,\mathrm{s}^{-1},-43.1^\circ,-196.2^\circ\,\mathrm{s}^{-1})$. (\textbf{Bottom}) Measured and nODE-predicted lower-arm angle $\beta$ for trajectory 99, with initial state $(-60.4^\circ,814.6^\circ\,\mathrm{s}^{-1},34.4^\circ,-202.4^\circ\,\mathrm{s}^{-1})$. Dotted vertical lines indicate the divergence time $t_{\mathrm{div}}$.}
    \label{fig:unstable_prediction}
\end{figure}

\subsection{Robot inverse dynamics}

\begin{table}[h]
\centering
\small
\setlength{\tabcolsep}{6pt}
\renewcommand{\arraystretch}{1.3}
\caption{Inverse-dynamics prediction performance on the industrial robot benchmark. The reported $R^2$ is averaged across the six robot joints.}
\label{tab:robot_r2}
\begin{tabular}{|l|c|}
\hline
Model & Overall $R^2$ (\%) \\
\hline
dLNN   & 88.8 \\
MLP    & 88.1 \\
gLNN   & 74.4 \\
Linear & 62.0 \\
LNN    & 57.5 \\
\hline
\end{tabular}
\end{table}

\begin{figure}[hbt]
    \centering
    \includegraphics[width=\linewidth]{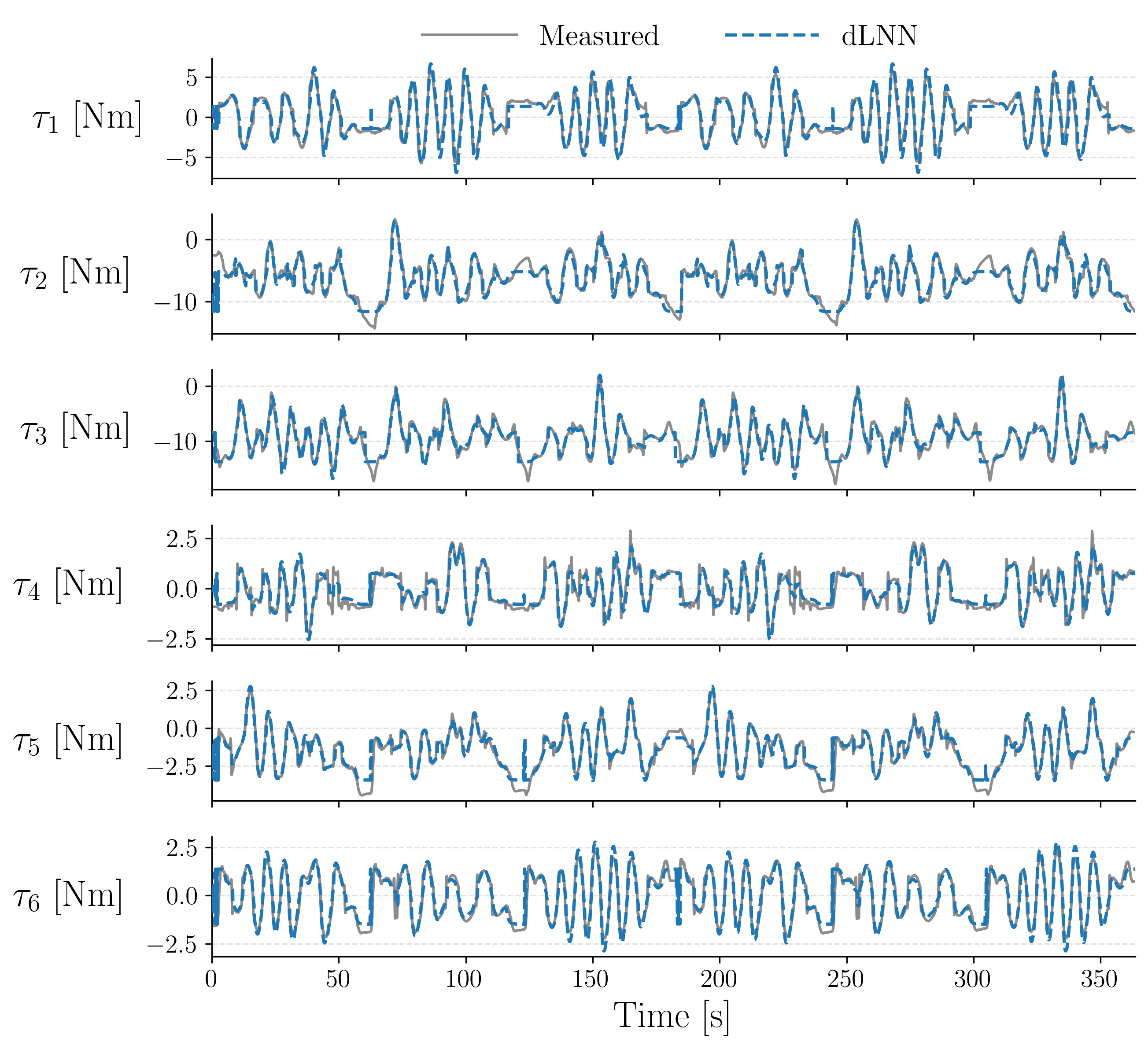}
    \caption{Measured and dLNN-predicted joint torques $\tau_1$–$\tau_6$ across the six sequential test trajectories of the industrial robot benchmark.}
    \label{fig:robot_prediction}
\end{figure}

Table~\ref{tab:robot_r2} summarises the inverse-dynamics performance on the robot benchmark. The dLNN achieves the highest torque-prediction performance, with an average $R^2$ of $88.8\%$ across the six robot joints, closely followed by the black-box MLP at $88.1\%$. The remaining models achieve lower prediction accuracy, with the gLNN, linear baseline and LNN reaching average $R^2$ values of $74.4\%$, $62.0\%$ and $57.5\%$, respectively.

Figure~\ref{fig:robot_prediction} shows the dLNN torque predictions across the test trajectories. Prediction errors are most pronounced during the low-velocity periods between the six sequential test trajectories, consistent with the difficulty of these operating regions reported for the original benchmark~\cite{weigand_dataset_2022}. Closer agreement between the measured and predicted torques is observed during the dynamic portions of the trajectories.

\section{Symbolic distillation}
\label{sec:distillation}
The property that the dLNN architecture can effectively decompose system dynamics into physical quantities introduces the possibility of distilling interpretable symbolic models from the learned functions $\mathbf{M}_\theta(\mathbf{q})$, $V_\theta(\mathbf{q})$ and $\mathbf{Q}^\mathrm{diss}_\theta(\mathbf{q}, \dot{\mathbf{q}})$. Accurate closed-form expressions, particularly for the conservative dynamics, could then replace the corresponding black-box neural networks while retaining the predictive model in an explicitly interpretable form.

Finding symbolic forms for $\mathbf{M}_\theta$ and $V_\theta$ would also remove the necessity to compute the derivatives $\nabla_\mathbf{q} V_\theta$ and 
$
\mathbf c_\theta(\mathbf q,\dot{\mathbf q})
=
\dot{\mathbf M}_\theta(\mathbf q)\dot{\mathbf q}
-
\frac12
\nabla_\mathbf{q}
\left(
\dot{\mathbf q}^{\mathsf T}
\mathbf M_\theta(\mathbf q)
\dot{\mathbf q}
\right)
$ using autodifferentiation at inference time since closed-form expressions for the terms $c_\theta(\mathbf q,\dot{\mathbf q})$ and $\nabla_\mathbf{q} V_\theta$ would also be readily available. This has potential to decrease the inference time and memory overhead of the model.
\subsection{Distillation procedure}
To distil symbolic expressions from the neural-network outputs $\mathbf{M}_\theta(\mathbf{q})$, $V_\theta(\mathbf{q})$, and $\mathbf{Q}^\mathrm{diss}_\theta(\mathbf{q},\dot{\mathbf{q}})$, a set of input states is first passed through the trained networks to generate corresponding input-output pairs, for example $(\mathbf{q}_i^*,\mathbf{M}_\theta(\mathbf{q}_i^*))$. The sampled inputs should ideally be representative of the training distribution and may simply be taken from the original training dataset.

To fit a symbolic model to this data, sequentially thresholded least squares regression (STLSQ) is applied, exactly as in the SINDy algorithm~\cite{brunton_discovering_2016}, with a library of candidate functions appropriately chosen for each of $\mathbf{M}_\theta$, $V_\theta$ and $\mathbf{Q}^\mathrm{diss}_\theta$ (Table \ref{tab:distillation_libraries}).
\subsection{Robot model symbolic distillation}
For the inverse robotic dynamics, this technique was applied in an attempt to reduce the inference time of a trained dLNN model. It is known that both the true $\mathbf{M}(\mathbf{q})$ and $V(\mathbf{q})$ will contain primarily trigonometric terms; however, in the interest of creating an efficient lower-order model, polynomial functions are chosen as the library functions for both $\mathbf{M}_\theta(\mathbf{q})$ and $V_\theta(\mathbf{q})$. Polynomial libraries up to a degree of two for $\mathbf{M}_\theta(\mathbf{q})$ and three for $V_\theta(\mathbf{q})$ are used (Table~\ref{tab:distillation_libraries}).

At this point, it is also possible to enforce known dependencies into the symbolic libraries. For example, it is known that $V(\mathbf{q})$ will not depend on the base yaw rotation $q_1$ so terms involving this coordinate can be removed from the library of candidate functions. For a serial robot arm, it is also known that $M_{ij}$ is only a function of $(q_{\operatorname{min}(i,j)+1},\dots,q_n)$, where joints are numbered 1-6 from the base upwards~\cite{khalil_chapter_2002}. These known dependencies can therefore be used to remove physically inadmissible terms from the corresponding symbolic libraries before regression.

Symbolic models are also fitted to the outputs of $\mathbf{Q}^\mathrm{diss}_\theta(\mathbf{q},\dot{\mathbf{q}})$. This time the library for $Q_i$ is simply $[\operatorname{sgn}(\dot{q}_i),\dot{q}_i]$, representing a simple Coulomb-viscous friction model~\cite{olsson_friction_1998}. Adding more complex velocity terms, such as $|\dot{q}_i|\dot{q}_i$ and $\dot{q}_i^3$, and allowing dependence on the configuration $\mathbf{q}$ did not improve performance over the simple two-term library.

The exact extracted symbolic expressions are found in Appendix~\ref{sec:Distilled-Robot}.

\begin{table}[h]
\centering
\small
\setlength{\tabcolsep}{6pt}
\renewcommand{\arraystretch}{1.3}
\caption{Symbolic libraries used for distillation of the six-joint
robot model. Polynomial libraries include all distinct terms,
including powers and cross terms.}
\label{tab:distillation_libraries}
\begin{tabular}{|l|cc|}
\hline
Component & Candidate functions & Coordinate indices \\
\hline
Mass-matrix entry $M_{ij}$
& $1,\;q_k,\;q_kq_\ell$
& $k,\ell\in\{\min(i,j)+1,\ldots,6\}$ \\
Potential energy $V$
& $1,\;q_k,\;q_kq_\ell,\;q_kq_\ell q_m$
& $k,\ell,m\in\{2,\ldots,6\}$ \\
Dissipative force $Q_i$
& $\operatorname{sgn}(\dot q_i),\;\dot q_i$
& -- \\
\hline
\end{tabular}
\end{table}

\subsection{Accuracy and computational cost}
As a result of the symbolic distillation, two new models were created: one is a fully symbolic model, and the other is a hybrid model that has symbolic $\mathbf{M}$ and $V$ while retaining the neural network $\mathbf{Q}^\mathrm{diss}_\theta$. These two models are then evaluated for torque prediction on the test set as the original models were. Additional tests were made for the CPU evaluation time on batch sizes of $B=1$ and $B=64$. The results are detailed in Table~\ref{tab:robot_symbolic_performance}.

The hybrid model was found to be very accurate, with only a 0.4-percentage-point reduction in $R^2$ torque prediction for a 3.2$\times$ speedup in CPU evaluation time compared to the dLNN model. The fully symbolic model saw a more drastic 7.3-percentage-point reduction in $R^2$ torque prediction but evaluated 4.89$\times$ faster than the standard dLNN. The computational advantage was more pronounced for the larger batch size, $B=64$, where the hybrid model and symbolic model were 6.4$\times$ and 11.3$\times$ faster than the dLNN, respectively.

\begin{table}[htbp]
\centering
\small
\setlength{\tabcolsep}{6pt}
\renewcommand{\arraystretch}{1.3}
\caption{Inverse dynamics prediction accuracy and CPU evaluation time for the dLNN, distilled models and black-box MLP on the industrial robot benchmark. Evaluation times are reported for batch sizes $B=1$ and $B=64$.}
\label{tab:robot_symbolic_performance}
\begin{tabular}{|l|ccc|}
\hline
Model & $R^2$ (\%) & CPU time, $B=1$ (ms) & CPU time, $B=64$ (ms) \\
\hline
dLNN & 88.8 & 0.494 & 3.162 \\
Hybrid ($M,V$ symbolic; $Q$ neural) & 88.4 & 0.155 & 0.491 \\
Fully Symbolic & 81.5 & 0.101 & 0.280 \\
MLP & 88.1 & 0.191 & 0.806 \\
\hline
\end{tabular}
\end{table}
Note that the reported evaluation times were measured using compiled PyTorch models in a Google Colab environment. Timings will depend significantly on hardware, implementation, network size, and symbolic model complexity, so these results should be viewed as illustrative rather than universally representative.

\section{Discussion}

\subsection{Physical consistency and predictive performance}
The advantage of the dLNN architecture, in terms of performance, comes from its improved ability to accurately generalise to unseen data. The inductive bias of enforcing a passive, mechanical structure greatly reduces the space of learnable models, improving the ability of the model to generalise to previously unseen parts of its input space. By enforcing physically realistic behaviour structurally rather than through the loss function, passivity is guaranteed even for dynamics outside the training distribution. This structure allows the dLNN model to generalise better than more flexible black-box models from limited amounts of data.

The degradation of the LNN model performance under even small amounts of friction also highlights an important point. For the simulated double pendulum, even under 5\% of the nominal friction, the LNN becomes the joint worst model, alongside the black-box nODE. At this level of friction, the ratio of the average frictional torque to the gravitational torque is estimated at 1:50,000 ($\left\|\mathbf{Q}^{\mathrm{diss}}\right\|_2
:
\left\|\nabla_{\mathbf{q}}V\right\|_2$) averaged across the training dataset. For this albeit idealised simulated system, this is enough to remove all performance advantage gained by the prior of exact energy conservation over the completely black-box nODE model. This highlights the importance of modelling dissipation alongside conservative mechanics in real-world systems, where even small energy losses can affect the dynamics. Similar problems are observed for the real-world double pendulum where the LNN performs notably worse than gLNN and dLNN, and for the robotic task the LNN performs even worse than the baseline linear model.

\subsection{Passivity and physical consistency}
A notable observation from the nODE predictions for the real double pendulum task was the possibility of the model to predict non-physical unstable dynamics. This is of particular concern for predictions in safety-critical applications, where strong average performance alone is insufficient and the behaviour of a model under worst-case conditions can determine whether it can be deployed safely~\cite{robey_probabilistically_2022}. By the passive construction of the dLNN network, unnatural injections of energy are prohibited in the learned dynamics. Even outside the training range, unforced trajectories predicted by the dLNN are guaranteed to settle towards stationary equilibria. For sparse data regions, the dLNN is observed to over-predict the dissipation present. This seems to be a much less extreme failure mechanism in the case of uncertainty than the instability that is sometimes observed for the black-box nODE model.

The passivity property is of particular importance if the learned model is to be used within a simulation environment for learning a reinforcement learning policy. If a simulation environment contains unphysical behaviour, such as artificial energy injection, a reinforcement-learning policy may learn to exploit these model inaccuracies rather than behaviours that transfer to the physical system. This is known as model exploitation in the reinforcement-learning literature~\cite{ibarz_how_2021}.

\subsection{Separation and interpretability}

A consequence of the dLNN formulation is that the learned non-conservative branch cannot freely absorb dynamics that should be represented by the Lagrangian. This helps retain the physical inductive bias of the model.

By contrast, an unconstrained generalised-force model can in principle bypass the Lagrangian structure entirely. For example, a gLNN may learn $\mathcal{L}_\theta=\frac{1}{2}\dot{\mathbf{q}}^\mathsf{T}\dot{\mathbf{q}}\leftrightarrow \mathbf{M}_\theta=\mathbf{I},\,V_\theta=0$, such that the unforced forward dynamics reduce to $\ddot{\mathbf{q}}=\mathbf{Q}_\theta(\mathbf{q},\dot{\mathbf{q}})$. The generalised-force network can then represent the complete acceleration field directly, making the model equivalent to a black-box second-order nODE in terms of the space of dynamics it can learn. 

Constraining $\mathbf{Q}^{\mathrm{diss}}_\theta(\mathbf{q},\mathbf{0})=\mathbf{0}$ prevents this degeneracy such that the dissipation network of the dLNN cannot absorb conservative dynamics. Therefore the dissipative extension retains the interpretability of the constituent kinetic energy, potential energy and dissipative terms. This enables the model to be used, as the original Deep Lagrangian Network~\cite{lutter_deep_2019} is, in energy-based control loops that directly use the readily accessible kinetic and potential energy values in a feedback control system~\cite{lutter_deep_2019-1}.

This separability also gives the dLNN a naturally modular structure, allowing $\mathbf{M}_\theta$, $V_\theta$, and $\mathbf{Q}^\mathrm{diss}_\theta$ to be inspected, modified, or retrained independently where appropriate. For example, the dissipative behaviour of a mechanical system may change over time due to wear~\cite{khonsari_running-nature_2021} while its conservative mechanics remain approximately fixed. In this case, $\mathbf{Q}^\mathrm{diss}_\theta$ could be updated without relearning the complete dynamics. 

It is also possible to then include auxiliary inputs that may only affect the dissipation network $\mathbf{Q}^\mathrm{diss}_\theta$. For example, the dissipation network could take temperature and humidity as auxiliary inputs~\cite{olsson_friction_1998}. This provides an avenue to include system dependence on these factors without altering the underlying conservative dynamics. It would even be possible to include a delay embedding into $\mathbf{Q}^\mathrm{diss}_\theta$ to allow the model to approximate dissipative behaviour with memory or hysteresis while retaining passivity, although achieving this without allowing the dissipation network to absorb conservative mechanics remains an open problem.

\subsection{Symbolic distillation}
The separability property of the dLNN enables us to attempt to distil closed-form expressions for $\mathbf{M}_\theta$, $V_\theta$ and $\mathbf{Q}^\mathrm{diss}_\theta$. The accuracy of the symbolic and hybrid (symbolic $\mathbf{M}$ and  $V$ and neural network $\mathbf{Q}^{\mathrm{diss}}_\theta$) models indicates that distilling a trained dLNN model into a simple closed-form expression is an effective way to transform a black-box model into an interpretable, computationally efficient symbolic model without excessively sacrificing performance. It is clear, however, that the symbolic distillation seems more appropriate for $\mathbf{M}$ and $V$ than for $\mathbf{Q}^\mathrm{diss}$. This seems to align with the general expectation that $\mathbf{M}$ and $V$ will be comparatively smooth and continuous, compared to the discontinuity and general complexity of real friction behaviour \cite{olsson_friction_1998}. The hybrid model, with symbolic $\mathbf{M}$ and $V$ but neural network $\mathbf{Q}^\mathrm{diss}_\theta$, seems to be a good middle ground in terms of overall performance, interpretability and computational performance. 

This hybrid formulation also suggests a natural grey-box extension of the dLNN. Where reliable first-principles expressions for $\mathbf{M}$ and $V$ are already available, these may be prescribed directly while only the uncertain dissipative term $\mathbf{Q}^{\mathrm{diss}}_\theta$ is learned from data, with the same passive formulation used for the dLNN. This may be particularly useful for mechanical systems in which the conservative dynamics are well understood but frictional and other dissipative effects are difficult to model accurately. In other words, this amounts to learning a passive discrepancy relative to a known conservative model.

\subsection{Future work}

Several directions remain for extending the proposed framework. Two key extensions might be:

\begin{itemize}
    \item \textbf{Richer dissipative dynamics}: The dLNN architecture could be extended to model more complex forms of dissipation. Interesting extensions could work towards considerations of temperature dependence, hysteretic friction, or other memory-dependent forces. A further extension should also focus on a more rigorous consideration of stick-slip friction.
    
    \item \textbf{Structure-preserving distillation}: Currently, the symbolic distillation procedure does not by construction enforce the positive definiteness of $\mathbf{M}$ and lower-boundedness of $V$. Both of these are required for the passivity/stability guarantee of the learned dynamics. An interesting extension could be to explore a method to distil symbolic expressions for $\mathbf{M}$ and $V$ such that these properties are guaranteed in the symbolic dynamics.
\end{itemize}

\section{Conclusion}

This work introduced the Dissipative Lagrangian Neural Network (dLNN), extending Lagrangian Neural Networks to dissipative mechanical systems while retaining a physically structured and interpretable representation of the dynamics. By explicitly separating conservative and dissipative components and constraining the latter to be passive, the dLNN combines physical consistency with the flexibility required to model real passive mechanical systems. Across simulated and experimental double-pendulum benchmarks, the dLNN provided consistently strong trajectory and energy predictions, particularly as dissipation increased and the conservative LNN assumption became inaccurate. On the industrial robot benchmark, the dLNN achieved the highest inverse-dynamics prediction accuracy while remaining competitive with a black-box neural model. Finally, the separation of the learned dynamics enabled symbolic distillation of the conservative and dissipative components into compact analytical expressions. These results demonstrate that explicitly incorporating dissipation into Lagrangian Neural Networks provides a practical route towards learned mechanical models that combine predictive accuracy, physical guarantees and interpretability.

\ack{The authors would like to thank John Aldridge and the technical staff at the University of Cambridge Department of Engineering (CUED) Mechanics Workshop for the design and custom fabrication of the double pendulum experimental apparatus, along with the undergraduate research project contributors for developing the data acquisition frameworks and the repository used in this study.}

\bibliographystyle{unsrt}
\bibliography{dlnn_BIB8}

\appendix

\section{Double pendulum equations of motion}
\label{sec:DP-EOM}

The double pendulum comprises two uniform rods and a point mass
$m_p$ at the intermediate pivot. The absolute rod angles $\alpha$
and $\beta$ are measured from the downward vertical. The Lagrangian is
\begin{equation}
    \mathcal{L} = \frac{1}{2}A\dot{\alpha}^{2}
      + \frac{1}{2}B\dot{\beta}^{2}
      + C\dot{\alpha}\dot{\beta}\cos(\alpha-\beta)
      + D\cos\alpha + E\cos\beta
\end{equation}
where $A=(m_p+m_a/3+m_b)a^2$, $B=m_b b^2/3$,
$C=m_b ab/2$, $D=ga(m_p+m_a/2+m_b)$ and $E=gm_b b/2$.
The physical parameters are listed in Table~\ref{tab:pendulum_parameters}. The total energy is 
\begin{equation}
\label{eq:ENERGY}
        \mathrm{Energy} = \frac{1}{2}A\dot{\alpha}^{2}
      + \frac{1}{2}B\dot{\beta}^{2}
      + C\dot{\alpha}\dot{\beta}\cos(\alpha-\beta)
      - D\cos\alpha - E\cos\beta
\end{equation}

\begin{table}[htbp]
\centering
\small
\setlength{\tabcolsep}{6pt}
\renewcommand{\arraystretch}{1.3}
\caption{Physical and friction parameters of the double-pendulum model, selected to reproduce the experimental apparatus.}
\label{tab:pendulum_parameters}
\begin{tabular}{|l|ccl|}
\hline
Quantity & Symbol & Value & Unit \\
\hline
Upper rod length & $a$ & $0.185$ & $\mathrm{m}$ \\
Lower rod length & $b$ & $0.172$ & $\mathrm{m}$ \\
Upper rod mass & $m_a$ & $0.044$ & $\mathrm{kg}$ \\
Lower rod mass & $m_b$ & $0.019$ & $\mathrm{kg}$ \\
Intermediate point mass & $m_p$ & $0.022$ & $\mathrm{kg}$ \\
Gravitational acceleration & $g$ & $9.81$ & $\mathrm{m\,s^{-2}}$ \\
First joint friction torque scale & $T_1$
    & $7.50\times10^{-4}$ & $\mathrm{N\,m}$ \\
Second joint friction torque scale & $T_2$
    & $2.50\times10^{-4}$ & $\mathrm{N\,m}$ \\
First joint viscous coefficient & $b_1$
    & $2.00\times10^{-5}$ & $\mathrm{N\,m\,s\,rad^{-1}}$ \\
Second joint viscous coefficient & $b_2$
    & $1.00\times10^{-5}$ & $\mathrm{N\,m\,s\,rad^{-1}}$ \\
Friction smoothing parameter & $\lambda$
    & $1.00$ & $\mathrm{s\,rad^{-1}}$ \\
\hline
\end{tabular}
\end{table}
To model joint dissipation, a friction force is modelled that combines viscous friction with a smooth Coulomb-type contribution, where the usual sign function $\operatorname{sgn}{v}$ is approximated with $\tanh{\lambda v}$:
\begin{align}
    \tau_1 &= -b_1\dot{\alpha}
           -T_1\tanh(\lambda\dot{\alpha}) \\
    \tau_2 &= -b_2(\dot{\beta}-\dot{\alpha})
           -T_2\tanh\!\left(\lambda(\dot{\beta}-\dot{\alpha})\right)
\end{align}
The second joint torque acts on the relative rotation
$\beta-\alpha$, giving the generalised forces
\begin{equation}
    Q_\alpha=\tau_1-\tau_2
    \qquad
    Q_\beta=\tau_2
\end{equation}
The Euler-Lagrange equations are solved and rearranged to give
the angular accelerations explicitly:

\begin{equation}
    \ddot{\alpha}
    =
    \frac{
        B\left[-D\sin\alpha-C\dot{\beta}^{2}\sin(\alpha-\beta)+Q_\alpha\right]
        -C\cos(\alpha-\beta)
        \left[-E\sin\beta+C\dot{\alpha}^{2}\sin(\alpha-\beta)+Q_\beta\right]
    }{
        AB-C^{2}\cos^{2}(\alpha-\beta)
    }
    \label{eq:double_pendulum_alpha}
\end{equation}

\begin{equation}
    \ddot{\beta}
    =
    \frac{
        A\left[-E\sin\beta+C\dot{\alpha}^{2}\sin(\alpha-\beta)+Q_\beta\right]
        -C\cos(\alpha-\beta)
        \left[-D\sin\alpha-C\dot{\beta}^{2}\sin(\alpha-\beta)+Q_\alpha\right]
    }{
        AB-C^{2}\cos^{2}(\alpha-\beta)
    }
    \label{eq:double_pendulum_beta}
\end{equation}
For small amplitudes, the natural frequencies of the system are
\begin{equation}
    \omega_{1,2}
    =
    \sqrt{
        \frac{
            AE+BD
            \mp \sqrt{(AE-BD)^2+4C^2DE}
        }{
            2(AB-C^2)
        }
    }.
\end{equation}
Using the parameters in Table~\ref{tab:pendulum_parameters} gives
\begin{equation}
    \omega_1 \approx 6.794\,\mathrm{rad\,s^{-1}},
    \qquad
    \omega_2 \approx 12.226\,\mathrm{rad\,s^{-1}},
\end{equation}
or, equivalently,
\begin{equation}
    f_1=\frac{\omega_1}{2\pi}\approx 1.081\,\mathrm{Hz},
    \qquad
    f_2=\frac{\omega_2}{2\pi}\approx 1.946\,\mathrm{Hz}.
\end{equation}
These are notably orders of magnitude lower than the measurement rate of 90 Hz.

\section{Distilled robot expressions}
\label{sec:Distilled-Robot}

Approximate symbolic expressions for $\mathbf{M}_\theta$, $V_\theta$ and $\mathbf{Q}^{\mathrm{diss}}_\theta$ for the industrial robotic arm are fitted to input-output data pairs using the SINDy algorithm~\cite{brunton_discovering_2016} as described in Section~\ref{sec:distillation}. The following expressions are found:

\begin{align*}
M_{11} &= 14.4 + 2.77 q_2 - 1.91 q_2 q_3 - 0.710 q_3^2
          + 0.546 q_3 - 0.198 q_5^2 + 0.193 q_5 \\[4pt]
M_{12} &= -0.982 q_2 - 0.896 q_3 - 0.851 q_3^2 - 0.396
          - 0.234 q_2 q_3 - 0.138 q_5 - 0.108 q_3 q_5 \\[4pt]
M_{13} &= -0.874 q_2 q_3 + 0.428 q_2^2 - 0.293 q_3
          + 0.221 q_2 - 0.166 q_3^2 \\
       &\quad - 0.152 - 0.0707 q_5^2 - 0.0618 q_5
          - 0.0564 q_3 q_5 \\[4pt]
M_{14} &= 0.244 q_2 q_3 + 0.151 q_3 - 0.142 q_2^2
          + 0.116 q_3^2 + 0.0773 q_2 \\
       &\quad + 0.0580 + 0.0505 q_3 q_5 + 0.0489 q_5 \\[4pt]
M_{15} &= -0.305 q_2^2 - 0.0856 q_3^2 + 0.0838
          - 0.0352 q_2 q_3 + 0.0250 q_2 \\
       &\quad + 0.0226 q_3 - 0.0176 q_3 q_5 + 0.0110 q_5 \\[4pt]
M_{16} &= -0.287 q_2^2 - 0.240 q_2 q_3 - 0.239 q_3^2
          - 0.122 q_3 - 0.0860 q_2 \\
       &\quad + 0.0634 - 0.0461 q_3 q_5 - 0.0232 q_5
          - 0.0158 q_5^2 \\[4pt]
M_{22} &= 23.9 + 1.42 q_3^2 + 1.09 q_3 + 0.0741 q_4^2 \\[4pt]
M_{23} &= 1.06 q_3^2 + 0.736 + 0.529 q_3 q_5 + 0.327 q_5 \\[4pt]
M_{24} &= 0.807 q_3 + 0.634 q_3^2 + 0.161
          + 0.0719 q_5 + 0.0705 q_3 q_5 \\[4pt]
M_{25} &= 0.175 - 0.143 q_3 - 0.0740 q_5
          - 0.0457 q_3 q_5 - 0.0370 q_3^2 \\[4pt]
M_{26} &= 0.144 q_3^2 - 0.0925 + 0.0470 q_3
          + 0.0401 q_5 + 0.0318 q_3 q_5 - 0.0228 q_5^2 \\[4pt]
M_{33} &= 12.1 \\[4pt]
M_{34} &= 0.198 + 0.00954 q_5 \\[4pt]
M_{35} &= 0.148 - 0.0425 q_5^2 + 0.0323 q_5 - 0.00771 q_4^2 \\[4pt]
M_{36} &= -0.149 + 0.0364 q_5^2 + 0.0238 q_5 + 0.00369 q_4^2 \\[4pt]
M_{44} &= 1.92 - 0.0137 q_5 \\[4pt]
M_{45} &= -0.0360 - 0.00843 q_5^2 \\[4pt]
M_{46} &= -0.0151 - 0.00472 q_5 \\[4pt]
M_{55} &= 3.18 \\[4pt]
M_{56} &= 0.0336 \\[4pt]
M_{66} &= 2.50 \\[8pt]
V &= -10.6 q_3 - 8.08 q_2 + 4.88 q_2^2 q_3
     + 4.88 q_2 q_3^2 - 3.73 + 3.55 q_2^2 \\
  &\quad + 1.90 q_2^3 - 1.85 q_5 + 1.68 q_3^3
     + 1.45 q_2 q_3 q_5 + 0.937 q_3 q_5^2 \\
  &\quad + 0.809 q_2 q_5^2 + 0.784 q_4^2 q_5
     + 0.632 q_3^2 q_5 + 0.236 q_3^2 + 0.236 q_5^3 \\
  &\quad - 0.109 q_3 q_5 + 0.101 q_3 q_4
     + 0.0859 q_3^2 q_4 + 0.0841 q_3 q_4^2 - 0.0391 q_4 \\[8pt]
Q_1 &= -3.99 \dot{q}_1 - 0.865 \operatorname{sgn}(\dot{q}_1) \\[4pt]
Q_2 &= -1.14 \operatorname{sgn}(\dot{q}_2) \\[4pt]
Q_3 &= -1.07 \operatorname{sgn}(\dot{q}_3) \\[4pt]
Q_4 &= -0.877 \dot{q}_4 - 0.342 \operatorname{sgn}(\dot{q}_4) \\[4pt]
Q_5 &= -1.77 \dot{q}_5 - 0.347 \operatorname{sgn}(\dot{q}_5) \\[4pt]
Q_6 &= -1.80 \dot{q}_6 - 0.362 \operatorname{sgn}(\dot{q}_6)
\end{align*}

\end{document}